\documentclass[letterpaper]{article} 
\usepackage[preprint]{aaai2027}  
\usepackage[hyphens]{url}  
\usepackage{graphicx} 
\usepackage{natbib}  
\usepackage{caption} 
\usepackage{amsmath}
\usepackage{amssymb}
\usepackage{booktabs}
\usepackage{array}
\usepackage{multirow}
\usepackage{makecell}
\usepackage[table]{xcolor}
\usepackage{algorithm}
\usepackage{algpseudocode}
\usepackage{pifont}
\usepackage{placeins}
\usepackage{amsthm}
\usepackage{mathtools}
\usepackage{enumitem}
\usepackage{float}
\usepackage{needspace}

\newcommand{\method}{VARI}
\newcommand{\GCD}{GCD}
\newcommand{\FCDPO}{FC-DPO}
\newcommand{\best}[1]{\textbf{#1}}
\newcommand{\second}[1]{\underline{#1}}
\newcommand{\cmark}{\ding{51}}
\newcommand{\xmark}{\ding{55}}
\newcommand{\Pz}{\mathcal{P}(z)}
\newcommand{\Xz}{\mathcal{X}(z)}
\newcommand{\clip}{\operatorname{clip}}
\definecolor{variGray}{gray}{0.91}
\definecolor{variLight}{gray}{0.96}
\definecolor{variBlue}{RGB}{226,238,252}
\definecolor{variBand}{RGB}{246,249,253}
\newcommand{\grayrow}{\rowcolor{variGray}}

\newcommand{\SummaryRowsVFourteen}{%
Base & 80.11 & 82.83 & 69.23 & 47.69 & 32.11 \\%
VideoDPO & 81.06 & 83.68 & 70.59 & 49.60 & 34.62 \\%
OnlineVPO & \second{81.60} & \second{84.19} & 71.24 & \second{50.49} & 35.31 \\%
\GCD{} & 81.15 & 83.46 & \second{71.91} & 50.48 & \second{35.73} \\%
\grayrow \method{} & \best{82.67} & \best{84.91} & \best{73.71} & \best{52.13} & \best{37.58} \\%
}

\newcommand{\TTwoVRowsVFourteen}{%
Base & 79.82 & 7.47 & 53.11 & 40.29 & 69.31 & 60.57 & 23.24 & 47.69 \\%
VideoDPO & 81.34 & 8.25 & 55.56 & \second{41.89} & 71.24 & 63.49 & \second{25.42} & 49.60 \\%
OnlineVPO & 82.31 & \best{10.04} & 56.72 & \best{43.91} & 72.13 & 64.49 & 23.81 & \second{50.49} \\%
\GCD{} & \second{83.26} & 8.66 & \second{56.97} & 39.88 & \second{72.71} & \second{66.61} & 25.29 & 50.48 \\%
\grayrow \method{} & \best{86.64} & \second{9.58} & \best{58.55} & 40.42 & \best{75.06} & \best{67.78} & \best{26.87} & \best{52.13} \\%
}

\newcommand{\HeadlineTTwoCGain}{1.64}
\newcommand{\TTwoRWorstGainPrompt}{2.45}
\newcommand{\VBenchRWorstGainPrompt}{3.55}
\newcommand{\TTwoRDropReductionPrompt}{1.21}
\newcommand{\VBenchRDropReductionPrompt}{2.77}
\newcommand{\TTwoRFIFRReductionPrompt}{3.78}
\newcommand{\VBenchRFIFRReductionPrompt}{2.99}
\newcommand{\TTwoRMeanGainGCD}{0.86}
\newcommand{\TTwoRWorstGainGCD}{0.72}

\newcommand{\TTwoRFIFRReductionGCD}{1.07}

\newcommand{\RobustnessRowsFinal}{%
Base & 47.18 & 44.47 & 4.92 & 11.14 & 68.29 & 61.27 & 11.74 & 11.38 & 82.78 \\%
VideoDPO & 49.54 & 47.08 & 3.97 & 8.93 & 70.16 & 65.18 & 9.19 & 8.72 & 83.51 \\%
OnlineVPO & 50.09 & 46.95 & 4.12 & 8.21 & 70.79 & 65.07 & 8.61 & 8.26 & \second{83.94} \\%
\GCD{} & \second{50.39} & \second{48.81} & \best{2.52} & \second{5.50} & \second{71.36} & \second{67.12} & \second{6.51} & \second{6.08} & 83.47 \\%
\grayrow \method{} & \best{51.25} & \best{49.53} & \second{2.76} & \best{4.43} & \best{72.68} & \best{68.73} & \best{5.84} & \best{5.27} & \best{84.24} \\%
}

\newcommand{\PEFullRowsFinal}{%
Raw & Base DiT & 47.69 & 44.47 & 4.92 & 11.14 \\%
Raw & \method{} DiT
& 52.13 & 49.53 & \best{2.76} & \second{4.43} \\%
Official PE & Base DiT
& \second{57.63} & \second{53.09} & 6.91 & 7.64 \\%
\grayrow Official PE & \method{} DiT
& \best{59.91} & \best{56.28}
& \second{4.57} & \best{4.14} \\
}

\newcommand{\LTXFullRowsVFourteen}{%
Base & 44.79 & 41.86 & 5.48 & 58.31 & 29.66 \\%
VideoDPO & 46.31 & 43.19 & 4.41 & 60.08 & 31.02 \\%
OnlineVPO & 46.84 & 43.04 & 4.82 & 60.67 & 31.48 \\%
+\GCD{} & \second{47.57} & \second{44.79} & \second{3.51} & \second{61.61} & \second{32.17} \\%
\grayrow +\method{} & \best{48.29} & \best{45.76} & \best{3.01} & \best{63.21} & \best{33.37} \\%
}

\newcommand{\ThresholdRowVFourteen}{%
$\tau_b$ (\%) & 78.07 & 7.18 & 50.22 & 36.31 & 65.09 & 56.37 & 20.19 \\
}

\newcommand{\SuppTrainingRowsVFourteen}{%
Reward-SFT & 48.58 & 43.82 & 5.97 & 13.11 \\%
Prompt-DPO & 49.91 & 46.77 & 4.49 & 9.24 \\%
FC-DPO & 50.31 & 47.68 & 4.01 & 7.76 \\%
\GCD{} & 50.46 & \second{48.88} & \second{2.57} & \second{5.41} \\%
\GCD{} + prompt-DPO & \second{50.82} & 48.64 & 3.28 & 6.19 \\%
\grayrow \method{} & \best{51.17} & \best{49.51} & \best{2.29} & \best{4.44} \\%
}

\newcommand{\MiningRowsVFourteen}{%
Single-form & \xmark & \xmark & \xmark & 41.08 & 16.89 & 14.94 \\%
Mean only & \cmark & \xmark & \xmark & 44.17 & 13.43 & 11.18 \\%
Mean + worst & \cmark & \cmark & \xmark & 46.12 & 10.21 & 8.24 \\%
Mean + worst - drop & \cmark & \cmark & \cmark & \best{46.95} & \best{7.48} & \best{5.96} \\%
\grayrow Full + $Q$ & \cmark & \cmark & \cmark & 46.87 & 7.64 & 6.08 \\
}

\newcommand{\PairScaleRowsVFourteen}{%
Random positive vs random negative & 43.12 & 15.79 & 14.09 \\%
Best-group vs low-score & 45.07 & 12.34 & 9.88 \\%
Best-group vs worst-group & 45.74 & 9.83 & 7.41 \\%
Best-group vs form-contrastive & \best{46.88} & \best{7.54} & 5.89 \\%
Top-$m$ positives + mixed negatives & 46.63 & 7.82 & \best{5.74} \\%
\midrule%
1 form, 4 seeds & 41.23 & 16.21 & 14.98 \\%
3 forms, 4 seeds & 44.93 & 10.71 & 9.37 \\%
5 forms, 2 seeds & 45.19 & 9.97 & 8.71 \\%
\grayrow 5 forms, 4 seeds & 46.81 & 7.73 & 6.07 \\%
7 forms, 4 seeds & \best{46.92} & \best{7.68} & \best{5.91} \\%
}

\newcommand{\VerificationRowsVFourteen}{%
Single original form & 49.91 & 41.24 & 16.23 & 14.98 \\%
Two verified forms & 50.55 & 43.87 & 11.82 & 10.69 \\%
Five unverified rewrites & 51.06 & 43.39 & 13.14 & 11.91 \\%
\grayrow Five verified forms & \best{51.76} & \best{46.81} & \best{7.74} & \best{6.04} \\%
}

\newcommand{\StoryRowsVFourteen}{%
StoryEval strict score & 32.11 & 33.21 & 34.62 & 35.31 & \best{37.58} \\
}

\providecommand{\method}{VARI}
\providecommand{\GCD}{GCD}
\providecommand{\FCDPO}{FC-DPO}
\providecommand{\Pz}{\mathcal{P}(z)}
\providecommand{\Xz}{\mathcal{X}(z)}
\providecommand{\best}[1]{\textbf{#1}}
\providecommand{\second}[1]{\underline{#1}}
\providecommand{\cmark}{\textcolor{green!50!black}{\ding{51}}}
\providecommand{\xmark}{\textcolor{red!70!black}{\ding{55}}}

\providecommand{\clip}{\operatorname{clip}}

\newcommand{\1}{\mathbf{1}}

\definecolor{variGray}{gray}{0.91}
\definecolor{variBlue}{RGB}{226,238,252}
\definecolor{variBand}{RGB}{246,249,253}
\definecolor{variRule}{RGB}{116,151,191}
\definecolor{variNavy}{RGB}{31,70,112}
\definecolor{variWarm}{RGB}{252,249,242}
\providecommand{\grayrow}{\rowcolor{variGray}}

\newtheorem{proposition}{Proposition}

\newsavebox{\promptboxsave}
\newenvironment{promptbox}[1]{%
  \begin{center}
  \setlength{\fboxsep}{7pt}%
  \begin{lrbox}{\promptboxsave}%
  \begin{minipage}{0.92\linewidth}%
  {\bfseries\color{variNavy}#1}\par\vspace{2pt}%
  {\color{variRule}\hrule height 0.7pt}\vspace{5pt}%
  \small
}{%
  \end{minipage}%
  \end{lrbox}%
  \fcolorbox{variRule}{variBand}{\usebox{\promptboxsave}}%
  \end{center}
}

\newsavebox{\compactpromptboxsave}
\newenvironment{compactpromptbox}[1]{%
  \begin{center}
  \setlength{\fboxsep}{4pt}%
  \begin{lrbox}{\compactpromptboxsave}%
  \begin{minipage}{0.96\linewidth}%
  {\bfseries\color{variNavy}#1}\par\vspace{1pt}%
  {\color{variRule}\hrule height 0.7pt}\vspace{2pt}%
  \footnotesize
}{%
  \end{minipage}%
  \end{lrbox}%
  \fcolorbox{variRule}{variBand}{\usebox{\compactpromptboxsave}}%
  \end{center}
}

\newsavebox{\analysisboxsave}
\newenvironment{analysisbox}[1]{%
  \begin{center}
  \setlength{\fboxsep}{6pt}%
  \begin{lrbox}{\analysisboxsave}%
  \begin{minipage}{0.93\linewidth}%
  {\bfseries #1}\par\vspace{2pt}%
  \small
}{%
  \end{minipage}%
  \end{lrbox}%
  \fcolorbox{black!35}{variWarm}{\usebox{\analysisboxsave}}%
  \end{center}
}

\providecommand{\SummaryRowsVFourteen}{%
Base & 80.11 & 82.83 & 69.23 & 47.69 & 32.11 \\%
VideoDPO & 81.06 & 83.68 & 70.59 & 49.60 & 34.62 \\%
OnlineVPO & \second{81.60} & \second{84.19} & 71.24 & \second{50.49} & 35.31 \\%
\GCD{} & 81.15 & 83.46 & \second{71.91} & 50.48 & \second{35.73} \\%
\grayrow \method{} & \best{82.67} & \best{84.91} & \best{73.71} & \best{52.13} & \best{37.58} \\%
}

\providecommand{\TTwoVRowsVFourteen}{%
Base & 79.82 & 7.47 & 53.11 & 40.29 & 69.31 & 60.57 & 23.24 & 47.69 \\%
VideoDPO & 81.34 & 8.25 & 55.56 & \second{41.89} & 71.24 & 63.49 & \second{25.42} & 49.60 \\%
OnlineVPO & 82.31 & \best{10.04} & 56.72 & \best{43.91} & 72.13 & 64.49 & 23.81 & \second{50.49} \\%
\GCD{} & \second{83.26} & 8.66 & \second{56.97} & 39.88 & \second{72.71} & \second{66.61} & 25.29 & 50.48 \\%
\grayrow \method{} & \best{86.64} & \second{9.58} & \best{58.55} & 40.42 & \best{75.06} & \best{67.78} & \best{26.87} & \best{52.13} \\%
}

\providecommand{\LTXFullRowsVFourteen}{%
Base & 44.79 & 41.86 & 5.48 & 58.31 & 29.66 \\%
VideoDPO & 46.31 & 43.19 & 4.41 & 60.08 & 31.02 \\%
OnlineVPO & 46.84 & 43.04 & 4.82 & 60.67 & 31.48 \\%
+\GCD{} & \second{47.57} & \second{44.79} & \second{3.51} & \second{61.61} & \second{32.17} \\%
\grayrow +\method{} & \best{48.29} & \best{45.76} & \best{3.01} & \best{63.21} & \best{33.37} \\%
}

\providecommand{\ThresholdRowVFourteen}{%
$\tau_b$ (\%) & 78.07 & 7.18 & 50.22 & 36.31 & 65.09 & 56.37 & 20.19 \\
}

\providecommand{\SuppTrainingRowsVFourteen}{%
Reward-SFT & 48.58 & 43.82 & 5.97 & 13.11 \\%
Prompt-DPO & 49.91 & 46.77 & 4.49 & 9.24 \\%
FC-DPO & 50.31 & 47.68 & 4.01 & 7.76 \\%
\GCD{} & 50.46 & \second{48.88} & \second{2.57} & \second{5.41} \\%
\GCD{} + prompt-DPO & \second{50.82} & 48.64 & 3.28 & 6.19 \\%
\grayrow \method{} & \best{51.17} & \best{49.51} & \best{2.29} & \best{4.44} \\%
}

\providecommand{\MiningRowsVFourteen}{%
Single-form & \xmark & \xmark & \xmark & 41.08 & 16.89 & 14.94 \\%
Mean only & \cmark & \xmark & \xmark & 44.17 & 13.43 & 11.18 \\%
Mean + worst & \cmark & \cmark & \xmark & 46.12 & 10.21 & 8.24 \\%
Mean + worst - drop & \cmark & \cmark & \cmark & \best{46.95} & \best{7.48} & \best{5.96} \\%
\grayrow Full + $Q$ & \cmark & \cmark & \cmark & 46.87 & 7.64 & 6.08 \\
}

\providecommand{\PairScaleRowsVFourteen}{%
Random positive vs random negative & 43.12 & 15.79 & 14.09 \\%
Best-group vs low-score & 45.07 & 12.34 & 9.88 \\%
Best-group vs worst-group & 45.74 & 9.83 & 7.41 \\%
Best-group vs form-contrastive & \best{46.88} & \best{7.54} & 5.89 \\%
Top-$m$ positives + mixed negatives & 46.63 & 7.82 & \best{5.74} \\%
\midrule%
1 form, 4 seeds & 41.23 & 16.21 & 14.98 \\%
3 forms, 4 seeds & 44.93 & 10.71 & 9.37 \\%
5 forms, 2 seeds & 45.19 & 9.97 & 8.71 \\%
\grayrow 5 forms, 4 seeds & 46.81 & 7.73 & 6.07 \\%
7 forms, 4 seeds & \best{46.92} & \best{7.68} & \best{5.91} \\%
}

\providecommand{\VerificationRowsVFourteen}{%
Single original form & 49.91 & 41.24 & 16.23 & 14.98 \\%
Two verified forms & 50.55 & 43.87 & 11.82 & 10.69 \\%
Five unverified rewrites & 51.06 & 43.39 & 13.14 & 11.91 \\%
\grayrow Five verified forms & \best{51.76} & \best{46.81} & \best{7.73} & \best{6.07} \\%
}

\providecommand{\StoryRowsVFourteen}{%
StoryEval strict score & 32.11 & 33.21 & 34.62 & 35.31 & \best{37.58} \\
}

\newcommand{\VATEXRowsVFourteen}{%
Base & 63.77 & 59.66 & 6.18 & 9.21 \\%
VideoDPO & 65.99 & 62.02 & 5.31 & 7.43 \\%
OnlineVPO & 66.47 & 61.92 & 6.14 & 7.75 \\%
\GCD{} & \second{67.29} & \second{64.49} & \second{4.17} & \second{5.33} \\%
\grayrow \method{} & \best{68.16} & \best{65.09} & \best{4.03} & \best{4.85} \\%
}

\newcommand{\MSRVTTRowsVFourteen}{%
Base & 61.63 & 57.17 & 6.67 & 10.22 \\%
VideoDPO & 63.61 & 59.33 & 6.07 & 8.76 \\%
OnlineVPO & 64.05 & 59.24 & 6.37 & 9.00 \\%
\GCD{} & \second{65.21} & \second{61.69} & \second{4.96} & \second{6.33} \\%
\grayrow \method{} & \best{65.88} & \best{62.34} & \best{4.69} & \best{5.84} \\%
}

\newcommand{\AuditRowsVThirtySix}{%
Group-consistent positive & 96/100 & 91/100 & 88/100 \\%
Low-score negative & 76/100 & 47/100 & 68/100 \\%
Dispersion-only negative & 58/100 & \best{81/100} & 50/100 \\%
Intent-audit negative & 85/100 & 76/100 & 78/100 \\%
\grayrow Audited FC negative & \best{91/100} & 79/100 & \best{84/100} \\%
}
\newcommand{\RunRowsVThirtyThree}{%
OnlineVPO & $50.49\!\pm\!0.23$ & $46.95\!\pm\!0.27$ & $65.07\!\pm\!0.29$ \\%
\GCD{} & $50.48\!\pm\!0.18$ & $48.81\!\pm\!0.21$ & $67.12\!\pm\!0.25$ \\%
\grayrow \method{} & $\best{52.13}\!\pm\!0.19$ & $\best{49.53}\!\pm\!0.22$ & $\best{68.73}\!\pm\!0.26$ \\%
}

\newcommand{\CoreAttributionRowsVThirtyFour}{%
\GCD{} & $50.46\!\pm\!0.18$ & $48.88\!\pm\!0.21$ & $2.57\!\pm\!0.13$ & $5.41\!\pm\!0.24$ \\%
Intent-audit-only FC-DPO & $50.94\!\pm\!0.15$ & $49.09\!\pm\!0.19$ & $2.61\!\pm\!0.12$ & $5.02\!\pm\!0.22$ \\%
\method{} w/o matched-seed pairs & $51.06\!\pm\!0.19$ & $49.28\!\pm\!0.18$ & $2.46\!\pm\!0.16$ & $4.87\!\pm\!0.20$ \\%
\grayrow \method{} (matched-seed prioritized) & $\best{51.17}\!\pm\!0.13$ & $\best{49.51}\!\pm\!0.16$ & $\best{2.29}\!\pm\!0.11$ & $\best{4.44}\!\pm\!0.18$ \\%
}
\newcommand{\SupervisionYieldRowsVThirtyFour}{%
At least one audited positive & 4,371 & 87.42 \\%
At least one audited negative & 3,228 & 64.56 \\%
At least one preference pair & 3,069 & 61.38 \\%
GCD-only intent & 1,302 & 26.04 \\%
No retained supervision & 629 & 12.58 \\%
\midrule
Matched-seed retained pairs & 4,236 & 71.71 \\%
Fallback retained pairs & 1,671 & 28.29 \\%
}
\newcommand{\OverlapRowsVThirtyThree}{%
VBench & 13 & 0 \\%
T2V-CompBench & 23 & 0 \\%
StoryEval & 8 & 0 \\%
VATEX & 17 & 0 \\%
MSR-VTT & 26 & 0 \\%
}
\newcommand{\DiversityRowsVThirtyThree}{%
Base & 100.00 & 100.00 & 11.36 \\%
OnlineVPO & 98.20 & 97.40 & 12.01 \\%
\GCD{} & 95.70 & 95.30 & 12.71 \\%
\grayrow \method{} & 96.50 & 95.80 & \best{13.90} \\%
}
\newcommand{\CounterfactualRowsVThirtyThree}{%
Count & 10.77 & 11.41 & 11.68 & \best{13.12} \\%
Spatial relation & 13.26 & 14.01 & 14.82 & \best{15.97} \\%
Action-role swap & 11.88 & 12.59 & 13.37 & \best{14.53} \\%
Temporal order & 9.52 & 10.04 & 10.96 & \best{11.98} \\%
}
\newcommand{\CausalRowsVThirtyThree}{%
Reward-SFT & 48.58 & 43.82 & 5.97 & 13.11 \\%
Prompt-level DPO & 49.91 & 46.77 & 4.49 & 9.24 \\%
Multi-form prompt-DPO & 50.08 & 47.03 & 4.34 & 8.61 \\%
Direct worst-form DPO & 50.39 & 48.17 & 3.12 & 6.38 \\%
Form-specific distillation & 49.93 & 47.51 & 3.69 & 7.19 \\%
Single-positive GCD & 50.18 & 48.56 & 2.89 & 5.93 \\%
\GCD{} & 50.46 & 48.88 & 2.57 & 5.41 \\%
GroupScore reranking (no train) & 49.69 & 46.45 & 4.13 & 8.69 \\%
Dispersion-only FC-DPO & 50.69 & 48.35 & 3.58 & 6.91 \\%
Intent-audit-only FC-DPO & 50.94 & 49.09 & 2.61 & 5.02 \\%
\method{} w/o matched-seed pairs & \second{51.06} & \second{49.28} & \second{2.46} & \second{4.87} \\%
\grayrow \method{} (matched-seed prioritized) & \best{51.17} & \best{49.51} & \best{2.29} & \best{4.44} \\%
}
\newcommand{\HumanDetailRowsVThirtySix}{%
Intent faithfulness & 352 & 96 & 152 & 58.67 & 16.00 & 25.33 & 33.33 & [25.9, 40.6] & 0.59 \\%
Cross-form consistency & 378 & 83 & 139 & 63.00 & 13.83 & 23.17 & 39.83 & [32.4, 47.1] & 0.63 \\%
Visual quality & 319 & 101 & 180 & 53.17 & 16.83 & 30.00 & 23.17 & [14.9, 31.3] & 0.57 \\%
Overall preference & 347 & 106 & 147 & 57.83 & 17.67 & 24.50 & 33.33 & [26.3, 40.4] & 0.60 \\%
}

\newcommand{\RewriteFunnelRowsVThirtySix}{%
Generated candidates & 53,487 & 100.00 \\%
Malformed / near-duplicate removed & 2,713 & 5.07 \\%
Field-wise verifier inputs & 50,774 & 94.93 \\%
Verifier-accepted candidates & 34,318 & 64.16 \\%
Retained rewrites (four per source) & 20,000 & 37.39 \\%
Complete five-form groups & 5,000 & 100.00 intents \\%
}
\newcommand{\EquivalenceAuditRowsVThirtySix}{%
Entity / count & 59/60 & 98.33 \\%
Spatial relation & 57/60 & 95.00 \\%
Action--role binding & 58/60 & 96.67 \\%
Temporal order & 56/60 & 93.33 \\%
Attribute / state / other & 58/60 & 96.67 \\%
\midrule
Overall & 288/300 & 96.00 \\%
}

\newcommand{\EquivalenceSourceRowsVThirtySeven}{%
Training accepted pairs & 190/200 & 95.00 \\%
Benchmark accepted pairs & 98/100 & 98.00 \\%
\midrule
All accepted pairs & 288/300 & 96.00 \\%
Complete benchmark groups (all four rewrites valid) & 97/100 & 97.00 \\%
}

\newcommand{\CommitmentYieldRowsVThirtySeven}{%
$J_z=3$--$4$ & 1,617 & 1,492 (92.3) & 913 (56.5) & 876 (54.2) \\%
$J_z=5$--$6$ & 2,356 & 2,068 (87.8) & 1,539 (65.3) & 1,474 (62.6) \\%
$J_z\ge 7$ & 1,027 & 811 (79.0) & 776 (75.6) & 719 (70.0) \\%
\midrule
All intents & 5,000 & 4,371 (87.4) & 3,228 (64.6) & 3,069 (61.4) \\%
}

\title{VARI: Robust Text-to-Video Generation via Group-Structured Intent Alignment}
\author{%
Mingxuan Zhu$^{1,2}$\equalcontrib\thanks{This work was conducted during the author's internship at KlingAI.}, Hohei Chan$^{2,3}$\equalcontrib, Shiyuan Liu$^{2}$,\\
Yuanxing Zhang$^{2}$, Lu Chen$^{1}$%
}
\affiliations{%
$^{1}$Shanghai Jiao Tong University\\
$^{2}$KlingAI\\
$^{3}$South China University of Technology%
}

\begin{document}
\maketitle

\begin{abstract}
Modern text-to-video (T2V) models generate visually compelling videos, yet meaning-preserving changes in prompt wording can substantially alter semantic realization and overall quality. Existing prompt rewriting modifies the inference input, whereas conventional post-training constructs supervision from one prompt formulation at a time; neither explicitly optimizes reliability across equivalent forms. We introduce \textbf{\method{}} (\textbf{V}erified \textbf{A}lignment over \textbf{R}epresentations of \textbf{I}ntent), a group-structured framework that treats automatically verified formulations of the same visual intent as a shared supervision unit. \method{} builds matched form--seed rollout pools, cross-scores each video against all forms, and applies a wording-independent structured intent audit to identify valid realizations and hard semantic failures. Group-Consistent Distillation (GCD) aligns every verified form with interchangeable valid realizations, while Form-Contrastive Flow-DPO (FC-DPO) learns group-level preferences from audited positive--negative pairs, prioritizing same-seed, different-form contrasts. On HunyuanVideo-1.5, \method{} attains the highest T2V-CompBench, VBench, and StoryEval scores among the evaluated methods. On prompt-form extensions of T2V-CompBench and VBench, it improves weakest-form alignment over the strongest prompt-level baselines by 2.45 and 3.55 points and reduces cross-form drop by 30.5\% and 32.2\%, respectively, while preserving visual quality. Human evaluation, natural-caption controls, prompt-enhancement experiments, diversity analyses, and transfer to LTX-2.3 further support the robustness and transferability of these gains, while \method{} requires no additional inference-time module.
\end{abstract}

\section{Introduction}

Large text-to-video generators based on diffusion or flow matching now achieve strong visual fidelity, motion quality, and instruction following
\cite{ddpm,flowmatching,hunyuanvideo}.
Yet robust semantic realization remains challenging: meaning-preserving changes in prompt wording can alter entity counts, spatial relations, action bindings, event realization, or overall generation quality.
Such variation is distinct from desirable visual diversity, since appearance, composition, and motion details may change while the core visual intent should remain invariant.
For example, consider ``A boy reading in front of a bench'' and ``A bench is positioned behind a boy who is reading.'' They express the same spatial relation, yet a generator may realize it under one wording and violate it under the other.

Inference-time prompt rewriting can improve a selected input, but it does not train the generator itself to behave reliably across equivalent formulations.
Reward- and preference-based post-training instead samples multiple rollouts for a prompt, scores them with automatic rewards, and learns from preferred and rejected examples
\cite{videodpo,dpo,onlinevpo,videoreward,dancegrpo}.
However, supervision constructed from one formulation at a time cannot distinguish intent-level robustness from wording-specific success.
Semantic-preserving prompt forms provide complementary linguistic evidence for constructing more reliable supervision.

\begin{figure*}[t]
\centering
\includegraphics[width=0.94\textwidth]{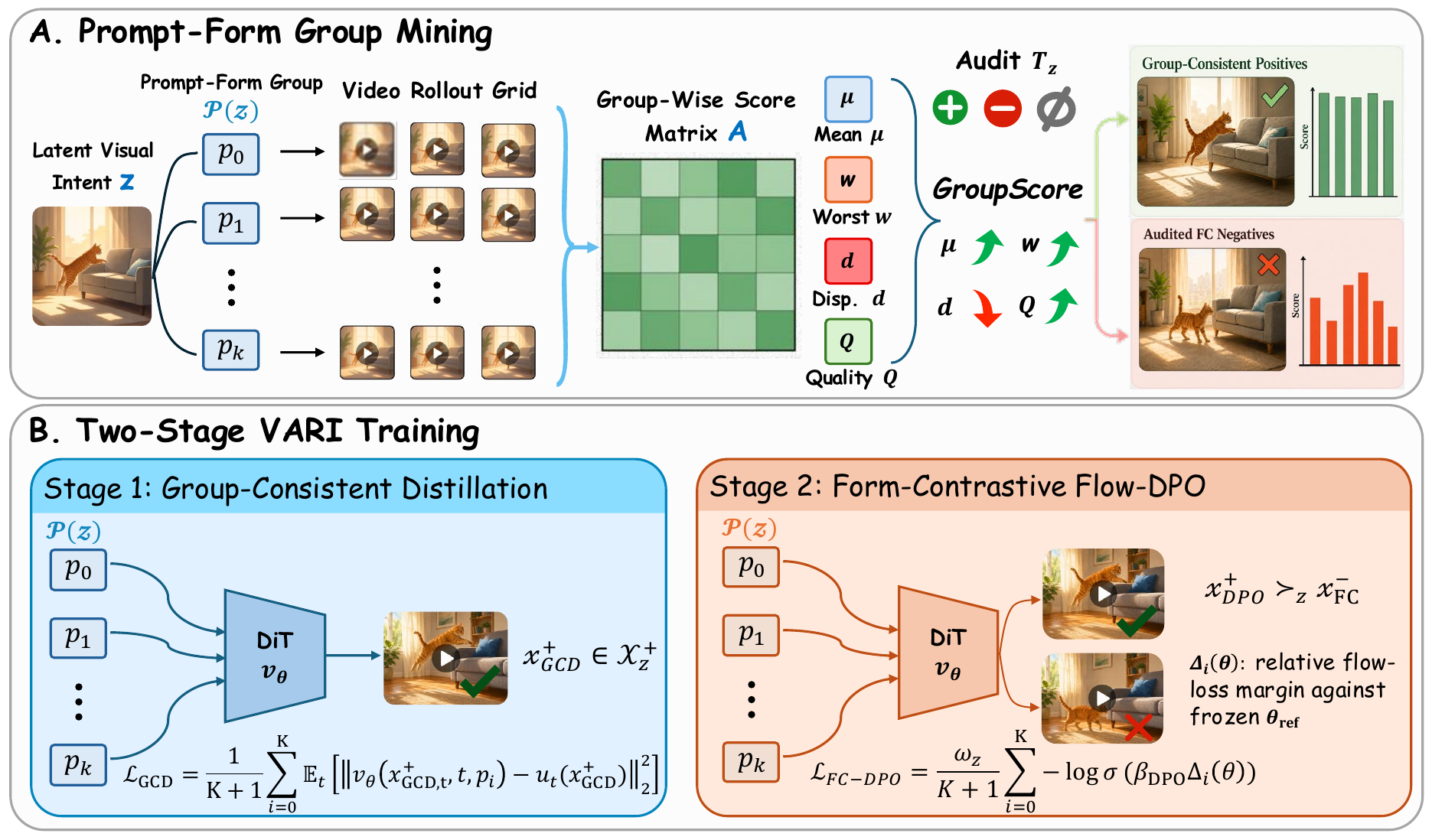}
\caption{Overview of \method{}. (A) Verified forms and shared seeds define a rollout grid. Cross-scoring provides mean support $\mu$, weakest-view support $w$, dispersion $d$, and quality $Q$; a structured audit against $T_z$ assigns semantic type, and GroupScore ranks candidates within type. (B) \GCD{} distills retained positives across all forms, while \FCDPO{} lifts audited intent-level preferences to the full group, prioritizing same-seed, different-form pairs.}
\label{fig:overview}
\end{figure*}

We introduce \textbf{\method{}} (\textbf{V}erified \textbf{A}lignment over \textbf{R}epresentations of \textbf{I}ntent), a group-structured intent-alignment framework that treats automatically verified prompt forms as a shared supervision unit.
For each visual intent $z$, we construct $\Pz=\{p_0,p_1,\ldots,p_K\}$, generate a matched form--seed rollout pool, and cross-score every candidate against all forms.
A generation seed determines the initial sampling noise; reusing it across forms controls stochastic variation, so same-seed differences more directly expose wording sensitivity.
Because equivalent forms share one structured semantic target, score variation for a fixed video is treated as \emph{scoring-view dispersion}, rather than form-dependent ground truth.
A frozen, wording-independent structured intent auditor checks atomic entities, counts, relations, action roles, and temporal commitments without observing prompt wording.
The audit determines semantic type, while cross-form evidence ranks candidates within type, yielding reliable positives and visually plausible hard negatives.

As illustrated in Figure~\ref{fig:overview}, \emph{Group-Consistent Distillation} (\GCD{}) trains every verified form on interchangeable valid realizations, whereas \emph{Form-Contrastive Flow-DPO} (\FCDPO{}) optimizes audited success--failure preferences under every form, prioritizing same-seed, different-form contrasts.
The resulting generator requires no additional inference-time module.
On HunyuanVideo-1.5~\cite{hunyuan15}, \method{} improves T2V-CompBench~\cite{t2vcompbench} by \HeadlineTTwoCGain{} points over the strongest prompt-level baseline and achieves the highest VBench~\cite{vbench} quality and StoryEval~\cite{storyeval} scores among the evaluated methods.
Prompt-form extensions of T2V-CompBench and VBench show weakest-form gains of \TTwoRWorstGainPrompt{} and \VBenchRWorstGainPrompt{} points, respectively.
Human evaluation and held-out controls show that the gains persist under auditor replacement, natural captions, prompt enhancement, and transfer to LTX-2.3 while preserving seed diversity and sensitivity to non-equivalent prompt changes.

Our main contributions are threefold:
\begin{itemize}
    \item \textbf{Verified prompt-form groups as supervision units.}
    We define prompt-form reliability over verified formulations of the same visually testable intent and use these groups, paired with shared seeds, as the supervision unit for T2V post-training.

    \item \textbf{Audited cross-form supervision.}
    We develop a compiler that cross-scores every rollout against all verified forms and uses a wording-independent structured intent audit to separate semantic validity from hardness ranking, yielding group-consistent positives and audited hard negatives.

    \item \textbf{Complementary group-level objectives.}
    We introduce \GCD{}, which trains every verified form on interchangeable valid realizations, and \FCDPO{}, which optimizes audited success--failure preferences under every form while prioritizing same-seed, different-form comparisons.
\end{itemize}

\section{Related Work}

\textbf{Diffusion, flow, and video generators.}
Modern visual generators build on denoising diffusion, latent diffusion, flow matching, rectified flows, diffusion transformers, and classifier-free guidance
\cite{ddpm,ldm,flowmatching,rectifiedflow,dit,cfg}.
Video generation has progressed from early pixel- and latent-space diffusion systems, cascaded spatial--temporal upsampling, and motion adaptation
\cite{videodiffusion,imagenvideo,makeavideo,videoldm,modelscope,animatediff,show1,videopoet,lumiere}
to large transformer- and flow-based backbones with stronger video latent spaces, including HunyuanVideo-1.5, LTX-2.3, and Wan2.2
\cite{svd,videocrafter2,latte,cogvideox,opensora,moviegen,hunyuanvideo,hunyuan15,ltxvideo,ltx23repo,wan,wan22repo}.
\method{} operates after pretraining and reorganizes supervision for an existing video DiT rather than modifying its backbone.

\textbf{Preference and reward-based alignment.}
Generative-model alignment includes direct preference objectives, policy-gradient fine-tuning, reward backpropagation, and reward-model-free denoising preference learning
\cite{dpo,ddpo,diffusiondpo,d3po}.
Image-generation studies have developed preference datasets, reward models, and input-side prompt optimization
\cite{pickapic,imagereward,hpsv2,promptist}.
For video, VADER, VideoDPO, VideoReward, OnlineVPO, DenseDPO, DanceGRPO, and VisionReward improve reward learning, preference construction, temporal granularity, or on-policy optimization
\cite{vader,videodpo,videoreward,onlinevpo,densedpo,dancegrpo,visionreward}.
Recent LocalDPO, Diffusion-APO, and Dual-IPO further study localized supervision, trajectory-aware updates, and iterative reward--generator alignment
\cite{localdpo,diffusionapo,dualipo}.
\method{} instead structures supervision across verified linguistic realizations of one visual intent. Main comparisons use VideoDPO and OnlineVPO because both admit matched HunyuanVideo-1.5 data and generation budgets; the newer methods change the underlying data or optimization pipeline.

\section{The VARI Framework}

\subsection{Semantic-Preserving Prompt-Form Groups}

Let $z$ denote a visual intent and let $p_0$ be its original prompt. A verified prompt-form group
\begin{equation}
\Pz=\{p_0,p_1,\ldots,p_K\}
\end{equation}
contains variants preserving the same structured target $T_z$, a set of atomic, visually testable commitments over objects, actions, attributes, relations, roles, and temporal order; unconstrained appearance may vary. Gemini 3.5 Flash generates ten training rewrites, Qwen3.6-27B verifies critical fields, and filtering retains four forms spanning lexical, syntactic, and caption-style variation; the roles are exchanged for benchmark groups. Human audit validates 190/200 training and 98/100 benchmark pairs, and 97/100 complete benchmark groups preserve $T_z$. Removing the three flagged groups changes headline robustness metrics by at most 0.11 points. Five forms and four shared seeds are adopted from held-out scale diagnostics, which show diminishing gains beyond five forms; details are in the supplement.

Each intent uses a seed set $\mathcal S_z$ shared across all verified forms, so form comparisons reuse the same initial sampling noises. The complete rollout pool is
\begin{equation}
\Xz=\{G_{\theta_{\rm roll}}(p_i,s):p_i\in\Pz,\;s\in\mathcal S_z\}.
\label{eq:pool}
\end{equation}
Here $G_{\theta_{\rm roll}}$ is the frozen rollout checkpoint used to construct supervision, and $s$ indexes the initial sampling noise. Mining uses the full within-intent form--seed pool; candidates are never paired across intents.

\subsection{Cross-Form Score Geometry}

Every candidate is evaluated against every verified form, yielding a score matrix $\mathbf A_z\in\mathbb R^{|\Xz|\times(K+1)}$ with
\begin{equation}
[\mathbf A_z]_{x,i}=A_i(x)
=\rho_A R_{\mathrm{VQA}}(x,p_i)+(1-\rho_A)R_{\mathrm{VS}}(x,p_i),
\label{eq:align}
\end{equation}
where VQAScore and VideoScore provide prompt-conditioned alignment signals \cite{vqascore,videoscore}. A separate low-weight quality term is
\begin{equation}
Q(x)=\rho_QR_{\mathrm{DOVER}}(x)+(1-\rho_Q)R_{\mathrm{motion}}(x),
\label{eq:quality}
\end{equation}
where $R_{\mathrm{motion}}$ summarizes motion smoothness, dynamics, flicker, and subject consistency \cite{dover,vbench}. Define the cross-form profile
{\small
\begin{equation}
\begin{aligned}
\mu(x,z)&=\frac{1}{K+1}\sum_{i=0}^{K}A_i(x),
& w(x,z)&=\min_i A_i(x),\\[-2pt]
b(x,z)&=\max_i A_i(x),
& d(x,z)&=b(x,z)-w(x,z),
\end{aligned}
\label{eq:stats}
\end{equation}
}
Thus $\mu$ measures average cross-form support, $w$ the weakest scoring view, $b$ the strongest, and $d$ their range. We rank candidates with $S_{\mathrm{grp}}=\lambda_{\mu}\mu+\lambda_w w-\lambda_d d+\lambda_qQ$, using $\rho_A=\rho_Q=0.6$ and validation-fixed $(\lambda_\mu,\lambda_w,\lambda_d,\lambda_q)=(0.50,0.25,0.15,0.10)$. Before the more expensive semantic audit, validation-fixed geometry--quality rules shortlist likely positives with high $w$, low $d$, and adequate $Q$, and likely hard negatives with high $b$, nontrivial $d$, and adequate $Q$. Only the deduplicated union is audited; the exact thresholds are in the supplement. This prefilter does not assign semantic labels: it only removes candidates that cannot satisfy the final positive or negative criteria.

For semantic typing, frozen Gemini 3.5 Flash (\texttt{gemini-3.5-flash}) \cite{gemini35flash} receives 24 temporally ordered frames and one atomic commitment $c_{z,j}$, without prompt strings, generating-form identity, or mining scores. A stateless score call returns probabilities for \texttt{A} (satisfied), \texttt{B} (violated), and \texttt{C} (uncertain); a separate diagnostic call records evidence only and never affects supervision. Write $T_z=\{c_{z,j}\}_{j=1}^{J_z}$, let $\mathcal K_z\subseteq\{1,\ldots,J_z\}$ index the required visually testable commitments, and define $s_j=P(\texttt{A})$ and $v_j=P(\texttt{B})$. We aggregate
\begin{equation}
\begin{aligned}
C_{\min}(x,z)&=\min_{j\in\mathcal K_z}s_j,
& C_{\mathrm{avg}}(x,z)&=\frac{1}{J_z}\sum_{j=1}^{J_z}s_j,\\[-2pt]
V_{\max}(x,z)&=\max_{j\in\mathcal K_z}v_j.&&
\end{aligned}
\label{eq:audit}
\end{equation}
A disjoint 600-video, three-annotator split fixes thresholds to achieve at least 90\% human precision; uncertain candidates abstain. The endpoint, request schema, frame sampling, and thresholds are frozen before mining, and the auditor is excluded from evaluation, checkpoint selection, and GroupScore tuning; Qwen3.6-27B \cite{qwen36} is used only for held-out replacement. Exact request and safety settings are in the supplement. Since all forms share $T_z$, $d$ is scoring-view dispersion, not form-dependent truth.

\subsection{Group-Structured Supervision Synthesis}

The shortlist supplies the necessary geometry and quality conditions, while the audit determines semantic type. With all components normalized to $[0,1]$, a stable positive must pass the positive shortlist and satisfy $C_{\min}\ge .70$ and $C_{\mathrm{avg}}\ge .80$; a hard negative must pass the negative shortlist and satisfy $V_{\max}\ge .75$. Uncertain candidates are excluded. We denote the resulting eligible sets by $\mathcal C_z^+$ and $\mathcal C_z^{\mathrm{FC}-}$, and retain $\mathcal X_z^+$ and $\mathcal X_z^{\mathrm{FC}-}$ as their top $m_+$ and $m_-$ candidates by GroupScore. Thus the audit fixes semantic type and GroupScore only ranks utility within type; for negatives, this favors visually plausible audit-confirmed failures rather than the most degraded videos. Exact set definitions are in the supplement.

Each rollout retains provenance $x_{i,s}=G_{\theta_{\rm roll}}(p_i,s)$. For a retained pair $r$, define its GroupScore gap as $g_{z,r}=S_{\mathrm{grp}}(x_{z,r}^+,z)-S_{\mathrm{grp}}(x_{z,r}^-,z)$ and its confidence weight as $\omega_{z,r}=\clip(g_{z,r},\omega_{\min},\omega_{\max})$. The compiler prioritizes audited positive--negative pairs generated with the same seed under different forms and satisfying a validation-fixed margin, then fills any remaining per-intent budget with top-gap audited pairs from the same intent; their union is $\Pi_z$. Thus ``form-contrastive'' describes generating-form provenance, while correctness is always judged against the common target $T_z$. The compiler emits one shared-target record per retained positive and group-lifted preference records for pairs in $\Pi_z$; exact set definitions and thresholds are in the supplement.

\subsection{Group-Consistent Distillation: Shared-Support Contraction}

For a video flow model with conditional velocity predictor $v_\theta$, let $\ell_{\mathrm{FM}}(x^+,p_i,t;\theta)$ denote the standard squared flow-matching error between $v_\theta(x_t^+,t,p_i)$ and target $u_t(x^+)$. Here $x_t^+$ is the corrupted or interpolated positive latent at sampled time $t$, and one record shares the latent and target across forms. Let $q_z^+$ sample the retained set $\mathcal X_z^+$; GCD optimizes
\begin{equation}
\mathcal L_{\mathrm{GCD}}
=\mathbb E_{z,\,x^+\sim q_z^+,\,t}
\left[\frac{1}{K+1}\sum_{i=0}^{K}\ell_{\mathrm{FM}}(x^+,p_i,t;\theta)\right].
\label{eq:gcd}
\end{equation}
Different records sample different valid positives and noise, so GCD contracts around a valid support set rather than one canonical video. The supplement shows that the per-record objective decomposes into fitting the group-mean prediction plus an explicit pairwise disagreement penalty; diversity and counterfactual controls test that this contraction does not erase valid variation or prompt sensitivity.

\subsection{Form-Contrastive Flow-DPO: Group-Lifted Margin Separation}

For a mined pair and sampled time $t$, let $\Delta_i(\theta,t)$ be the policy-versus-reference change in the rejected-minus-preferred flow-loss margin under form $p_i$, where the reference is the frozen post-GCD checkpoint. \method{} lifts the intent-level relation $x_{z,r}^+\succ_z x_{z,r}^-$ to every verified form and instantiates the preference update with a flow-DPO kernel:
\begin{align}
\mathcal L_{\mathrm{FC\mbox{-}DPO}}
&=\mathbb E_{z,r,t}\!\left[
\frac{\omega_{z,r}}{K+1}\sum_{i=0}^{K}
-\log\sigma\!\left(\beta_{\mathrm{DPO}}\Delta_i(\theta,t)\right)
\right]\nonumber\\[-2pt]
&\quad+\lambda_{\mathrm{aux}}\,
\mathbb E_{z,r,t}\!\left[
\frac{1}{K+1}\sum_{i=0}^{K}\ell_{\mathrm{FM}}(x_{z,r}^+,p_i,t;\theta)
\right].
\label{eq:fcdpo}
\end{align}
We set $\beta_{\mathrm{DPO}}=0.3$ and $\lambda_{\mathrm{aux}}=0.05$. Within each branch, the policy and reference evaluate the same sampled time and corrupted input. The first term separates positives from audited hard negatives under every form; because the nonlinear preference loss is averaged after each form-wise margin is computed, an easy form cannot cancel an unresolved form before the logistic penalty, and smaller margins receive stronger updates. The auxiliary term preserves positive support, while matched seeds reduce stochastic confounding.

\begin{algorithm}[t]
\caption{Audited Prompt-Form Supervision Compiler}
\label{alg:vari}
\scriptsize
\begin{algorithmic}[1]
\Require Intent $z$, target $T_z$, forms $\Pz$, shared seeds $\mathcal S_z$, budgets $m_+,m_-$, margin $\tau_{\mathrm{pair}}$, clip bounds $\omega_{\min},\omega_{\max}$
\Ensure GCD records $\mathcal D_{\mathrm{GCD}}^z$ and FC-DPO records $\mathcal D_{\mathrm{FC}}^z$
\For{each form $p_i\in\Pz$ and seed $s\in\mathcal S_z$}
    \State Generate $x_{i,s}$ and retain provenance $(i,s)$
\EndFor
\State Cross-score the pool to obtain $\mathbf A_z$, $Q$, and $(\mu,w,b,d,S_{\mathrm{grp}})$
\State Form geometry--quality shortlists $\mathcal U_z^+$ and $\mathcal U_z^-$ using validation-fixed conditions
\For{each $x\in\mathcal U_z^+\cup\mathcal U_z^-$ after deduplication}
    \State Audit atomic commitments in $T_z$ and assign positive, negative, or abstained type
\EndFor
\State Rank eligible candidates within type; retain $\mathcal X_z^+$ and $\mathcal X_z^{\mathrm{FC}-}$
\State $\mathcal D_{\mathrm{GCD}}^z\gets\{(\Pz,x^+):x^+\in\mathcal X_z^+\}$; $\mathcal D_{\mathrm{FC}}^z\gets\varnothing$
\For{matched-seed pairs first, then top-gap audited fallback pairs}
    \State $g\gets S_{\mathrm{grp}}(x^+,z)-S_{\mathrm{grp}}(x^-,z)$
    \If{$g\ge\tau_{\mathrm{pair}}$ and the budget is not full}
        \State Add $(\Pz,x^+,x^-,\clip(g,\omega_{\min},\omega_{\max}))$ to $\mathcal D_{\mathrm{FC}}^z$
    \EndIf
\EndFor
\State \Return $(\mathcal D_{\mathrm{GCD}}^z,\mathcal D_{\mathrm{FC}}^z)$
\end{algorithmic}
\end{algorithm}

\section{Experiments}
\label{sec:experiments}

\subsection{Experimental Setup}

\textbf{Backbone and training pool.} We use HunyuanVideo-1.5 \cite{hunyuan15} with a unified 480p, 16:9, 121-frame, 24-fps, 50-step protocol. The training pool contains 5,000 independently constructed source intents balanced across attributes, relations, actions, interactions, numeracy, and multi-event composition; each has four verified rewrites and four shared seeds, yielding 100,000 rollouts. All trainable methods use matched groups and generation budgets. Offline methods share the same pool, whereas OnlineVPO performs on-policy refresh under the same total budget. Main comparisons are Base, VideoDPO, OnlineVPO, \GCD{}, and full \method{}, with additional controls in the ablation.

\textbf{Evaluation principle.} Mining scores are not reused for reporting. Benchmark results are recomputed with the official VBench, T2V-CompBench, and StoryEval evaluator implementations. VBench and T2V-CompBench use the same 480p, 16:9, 121-frame, 24-fps, 50-step generation configuration for every system; StoryEval follows its official 10s setting. These extensions introduce no learned evaluator: they reuse the official per-video scores unchanged and add only verified prompt-form grouping and intent-level aggregation, yielding VBench-R and T2V-CompBench-R.

\textbf{Prompt-form robustness protocol.} For each of the 1,400 T2V-CompBench prompts, we construct five verified forms. VBench-R uses matched official semantic dimensions; each rewritten form inherits the source category and structured target, and only the generation-side wording changes. T2V-CompBench-R reuses the official score for the inherited category, whereas VBench-R reuses the matched official \mbox{semantic} dimension and reports official VBench quality separately on the same clips. Mean, Worst, and max--min Drop are aggregated within each intent. FIFR counts mixed-success groups in which at least one form passes and another fails a validation-fixed, category- or dimension-specific threshold; it is interpreted jointly with Mean and Worst rather than as standalone generation quality.

\textbf{Data independence and natural-caption controls.} Source intents are split before rewriting, so every form of an intent remains in one partition. VBench, T2V-CompBench, StoryEval, VATEX, and MSR-VTT prompts are excluded from training, calibration, and checkpoint selection. Normalized exact matching, lexical retrieval, and frozen sentence-embedding retrieval followed by manual review find no shared visual intent; training and benchmark groups also use disjoint rewrite instructions and verifier instances. The supplement reports the full overlap audit and held-out VATEX \cite{vatex} and MSR-VTT \cite{msrvtt} controls scored by VisionReward \cite{visionreward}. VATEX/MSR-VTT weakest-form scores improve from 61.92/59.24 for OnlineVPO to 65.09/62.34 for \method{}.

\textbf{Audit, coverage, and statistics.} The protocol is frozen before mining, and only the union of the validation-fixed positive and negative shortlists is audited. Across 5,000 intents, 87.4\% yield an audited positive, 64.6\% an audited negative, and 61.4\% a preference pair; 71.7\% of retained pairs are same-seed, different-form. VideoDPO, OnlineVPO, \GCD{}, and \method{} use three matched training seeds; full deviations and source-intent paired tests are reported in the supplement.

\textbf{Compute efficiency.} All trainable methods share the 100,000-video rollout budget. \method{} adds one-time, parallelizable offline cross-scoring, shortlist auditing, and a second optimization stage, giving modest overhead over VideoDPO and the same cost regime as OnlineVPO, while leaving per-video inference cost unchanged; details are in the supplement.

\subsection{Standard Benchmark Results}

Table~\ref{tab:summary} gives the headline benchmark view. VBench provides quality and semantic anchors, T2V-CompBench measures compositional prompt following, and StoryEval evaluates strict event completion \cite{storyeval}. \method{} gives the strongest overall scores. \GCD{} is comparatively stronger on semantic consistency, while the full method also improves quality and StoryEval, indicating complementary stages without trading away fidelity or event completion.

\begin{table}[t]
\centering
\scriptsize
\setlength{\tabcolsep}{1.35pt}
\renewcommand{\arraystretch}{0.94}
\begin{tabular}{@{}lccccc@{}}
\toprule
\textbf{Method} & \makecell{\textbf{VBench}\\\textbf{Total}} & \makecell{\textbf{VBench}\\\textbf{Qual.}} & \makecell{\textbf{VBench}\\\textbf{Sem.}} & \makecell{\textbf{T2V-C}\\\textbf{Avg.}} & \textbf{StoryEval} \\
\midrule
\SummaryRowsVFourteen
\bottomrule
\end{tabular}
\caption{Standard benchmark results. T2V-C averages seven official categories; StoryEval reports strict event completion.}
\label{tab:summary}
\end{table}

\subsection{Dimension-Wise Diagnostics}

Table~\ref{tab:dimensions} expands T2V-CompBench into all seven categories. OnlineVPO remains strongest on dynamic attributes and motion binding, whereas \method{} gains most on constant attributes, spatial relations, action binding, interaction, and numeracy. This pattern matches the entity, relation, role, and count commitments in the structured target rather than a uniform improvement in low-level motion generation.

\begin{table*}[!t]
\centering
\scriptsize
\setlength{\tabcolsep}{2.00pt}
\renewcommand{\arraystretch}{0.84}
\resizebox{0.86\textwidth}{!}{%
\begin{tabular}{@{}lcccccccc@{}}
\toprule
\textbf{Method} & \textbf{Cons. Attr.} & \textbf{Dyn. Attr.} & \textbf{Spatial Rel.} & \textbf{Motion Bind.} & \textbf{Action Bind.} & \textbf{Interaction} & \textbf{Numeracy} & \textbf{Avg.} \\
\midrule
\TTwoVRowsVFourteen
\bottomrule
\end{tabular}%
}
\caption{T2V-CompBench categories. \method{} gives the best average; OnlineVPO leads dynamic attribute and motion binding.}
\label{tab:dimensions}
\end{table*}

\begin{table*}[!t]
\centering
\scriptsize
\setlength{\tabcolsep}{1.85pt}
\renewcommand{\arraystretch}{0.85}
\resizebox{0.76\textwidth}{!}{%
\begin{tabular}{@{}lccccccccc@{}}
\toprule
\multirow{2}{*}{\textbf{Method}} & \multicolumn{4}{c}{\textbf{T2V-CompBench-R}} & \multicolumn{4}{c}{\textbf{VBench-R Semantic}} & \multicolumn{1}{c}{\textbf{VBench-R Quality}} \\
\cmidrule(lr){2-5}\cmidrule(lr){6-9}\cmidrule(l){10-10}
& Mean$\uparrow$ & Worst$\uparrow$ & Drop$\downarrow$ & FIFR$\downarrow$ & Mean$\uparrow$ & Worst$\uparrow$ & Drop$\downarrow$ & FIFR$\downarrow$ & Mean$\uparrow$ \\
\midrule
\RobustnessRowsFinal
\bottomrule
\end{tabular}%
}
\caption{Prompt-form group extensions of the official evaluators; VBench quality is reported separately.}
\label{tab:robustness}
\end{table*}

\subsection{Prompt-Form Robustness Extensions}

\begin{figure}[!b]
\centering
\includegraphics[width=0.92\linewidth]{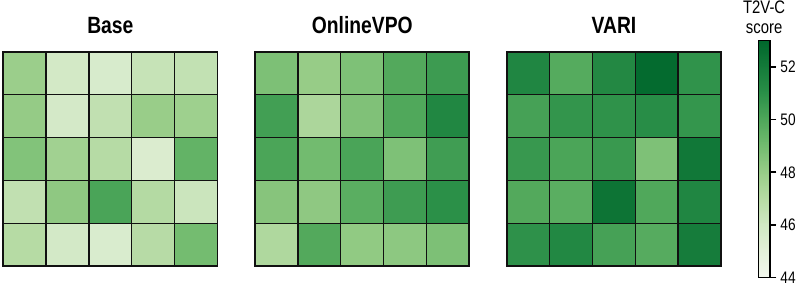}
\caption{Representative mining-view score matrices for one intent group. Rows are generated candidates and columns are verified scoring forms. Cell variation is diagnostic evidence, not form-dependent ground truth for a fixed video.}
\label{fig:heatmaps}
\end{figure}

Table~\ref{tab:robustness} is the central prompt-form robustness evaluation. Unlike the mining stage, the scores in this table are official benchmark scores aggregated over verified prompt-form groups. Let $s_i^{(b)}(z)$ denote the official score of benchmark $b$ for the video generated from the $i$-th prompt form of intent $z$. We report
\begin{equation}
\begin{aligned}
\mathrm{Mean}(z)&=\frac{1}{|\mathcal{P}(z)|}\sum_i s_i^{(b)}(z),\\[-1pt]
\mathrm{Worst}(z)&=\min_i s_i^{(b)}(z),\\[-1pt]
\mathrm{Drop}(z)&=\max_i s_i^{(b)}(z)-\min_i s_i^{(b)}(z).
\end{aligned}
\end{equation}
Form-Induced Failure Rate (FIFR) measures mixed-success groups:
\begin{equation}
\begin{aligned}
f_z^{(b)}&=\mathbf{1}\!\left[\min_i s_i^{(b)}(z)<\tau_b\ \wedge\ \max_i s_i^{(b)}(z)\ge\tau_b\right],\\[-1pt]
\mathrm{FIFR}&=\frac{100}{N}\sum_{z=1}^{N}f_z^{(b)}.
\end{aligned}
\label{eq:fifr}
\end{equation}
Here $\tau_b$ is fixed on held-out validation data for each inherited category or matched semantic dimension. FIFR is interpreted jointly with Mean and Worst because a uniformly weak model can also have low mixed-success inconsistency; VBench quality is reported separately as a preservation guardrail.

Against the strongest prompt-level baseline for each endpoint, \method{} raises T2V-CompBench-R/VBench-R Worst by \TTwoRWorstGainPrompt{}/\VBenchRWorstGainPrompt{} points, reduces Drop by \TTwoRDropReductionPrompt{}/\VBenchRDropReductionPrompt{}, and reduces FIFR by \TTwoRFIFRReductionPrompt{}/\VBenchRFIFRReductionPrompt{}. Relative to \GCD{}, it further improves T2V-R Mean/Worst by \TTwoRMeanGainGCD{}/\TTwoRWorstGainGCD{} and lowers FIFR by \TTwoRFIFRReductionGCD{}, while T2V-R Drop increases modestly from 2.52 to 2.76. Together, the results suggest that \GCD{} supplies the main cross-form contraction, whereas \FCDPO{} raises the lower tail and removes additional mixed-success failures, with only a small trade-off in the average within-group range. Intent-level VBench-R gains are significant (Worst 95\% CI $[2.66,4.43]$, Holm-adjusted $p<0.001$; FIFR McNemar $p<0.01$ after correction).

Figure~\ref{fig:heatmaps} shows a held-out cross-score matrix: OnlineVPO raises many cells but remains uneven, whereas \method{} is higher and more stable. These scores provide mining evidence only; the structured audit assigns semantic type against the shared target.

\subsection{Distribution Analysis}

Figure~\ref{fig:distributions} summarizes all VBench-R groups. \method{} shifts worst-form scores upward and tightens Drop; simultaneous Mean and Worst gains suggest that the improvement is not explained solely by a few strong forms or uniformly weak contraction.

\begin{figure}[!b]
\centering
\includegraphics[width=0.96\linewidth]{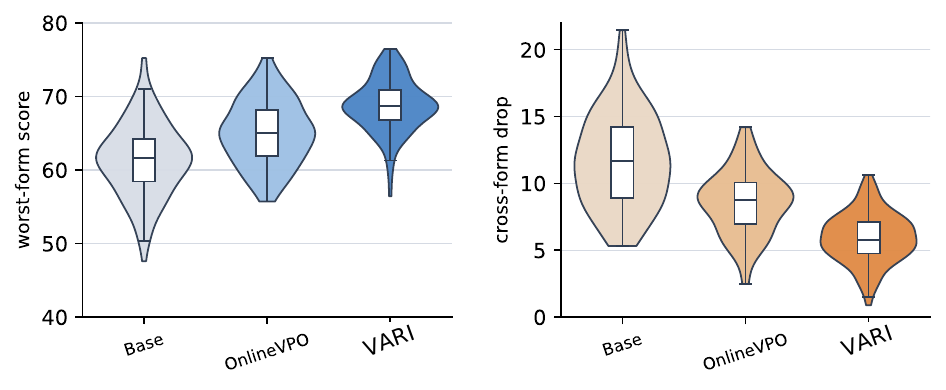}
\caption{Group-level VBench-R semantic distributions. \method{} raises worst-form scores and reduces cross-form drops.}
\label{fig:distributions}
\end{figure}

\begin{figure*}[!t]
\centering
\includegraphics[width=1.00\textwidth]{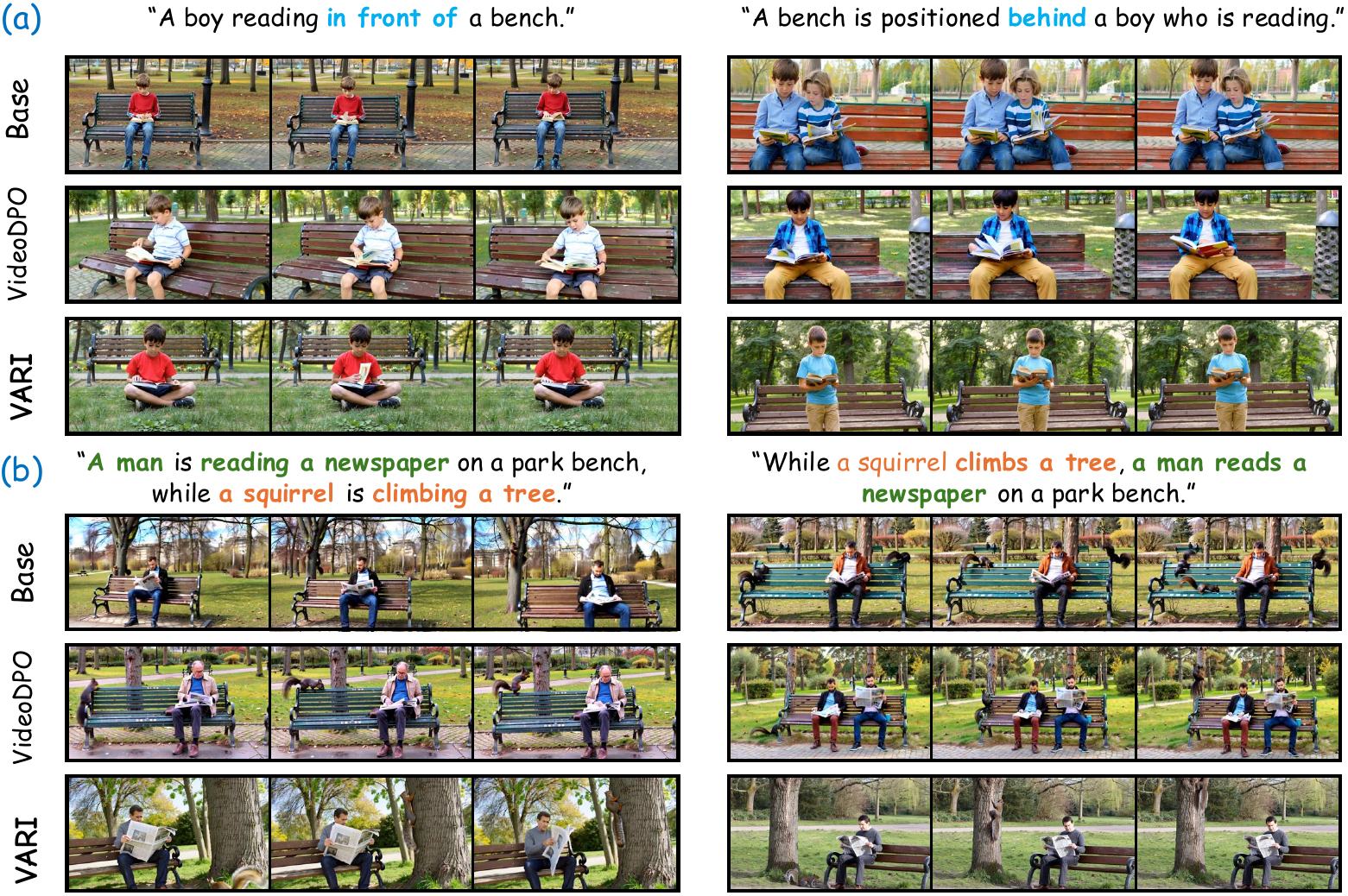}
\caption{Source/equivalent-form cases. Baselines violate spatial relation or action binding; \method{} preserves both across forms. The panels show temporally ordered frame strips from the evaluated video clips.}
\label{fig:qualitative-cases}
\end{figure*}

\subsection{Compatibility with Prompt Enhancement}

PE rewrites the input, whereas \method{} changes the generator. We freeze the official HunyuanVideo-1.5 template \cite{hunyuan15repo,hunyuan15handbook} for both methods and evaluate against the original forms. Under identical enhanced prompts, \method{} improves T2V-C/R-Worst from 57.63/53.09 to 59.91/56.28 and reduces R-Drop/FIFR from 6.91/7.64 to 4.57/4.14, isolating improved generator behavior from rewrite strength.

\subsection{Ablation Studies}

Table~\ref{tab:abl-train} controls paraphrase exposure, worst-form optimization, target sharing, audit, hardness, and pair provenance. Exposure alone is insufficient; \GCD{} provides the main contraction, while prompt-DPO after \GCD{} raises Mean but degrades the lower tail. Audited group preferences improve Worst, Drop, and FIFR jointly, with matched-seed \mbox{provenance} adding the final gain.

\begin{center}
\begin{minipage}{\linewidth}
\centering
\scriptsize
\setlength{\tabcolsep}{2.0pt}
\renewcommand{\arraystretch}{0.72}
\begin{tabular}{@{}lcccc@{}}
\toprule
\textbf{Variant} & Mean$\uparrow$ & Worst$\uparrow$ & Drop$\downarrow$ & FIFR$\downarrow$ \\
\midrule
Reward-SFT & 48.58 & 43.82 & 5.97 & 13.11 \\ 
Multi-form prompt-DPO & 50.08 & 47.03 & 4.34 & 8.61 \\ 
\GCD{} + prompt-DPO & 50.82 & 48.64 & 3.28 & 6.19 \\ 
Direct worst-form DPO & 50.39 & 48.17 & 3.12 & 6.38 \\ 
Form-specific distillation & 49.93 & 47.51 & 3.69 & 7.19 \\ 
Dispersion-only FC-DPO & 50.69 & 48.35 & 3.58 & 6.91 \\ 
\GCD{}$^{\dagger}$ & 50.46 & 48.88 & 2.57 & 5.41 \\ 
Intent-audit-only FC-DPO$^{\dagger}$ & 50.94 & 49.09 & 2.61 & 5.02 \\ 
\method{} w/o matched-seed pairs$^{\dagger}$ & \second{51.06} & \second{49.28} & \second{2.46} & \second{4.87} \\ 
\grayrow \method{} (matched-seed prioritized)$^{\dagger}$ & \best{51.17} & \best{49.51} & \best{2.29} & \best{4.44} \\
\bottomrule
\end{tabular}
\captionof{table}{Mechanism ablation on the fixed development split; $\dagger$ denotes three paired runs.}
\label{tab:abl-train}
\end{minipage}
\end{center}

\subsection{Supervision Validity and Boundary Controls}

The blind audit in Table~\ref{tab:validity} separates semantic validity from dispersion: audited FC negatives reach 91/100 defect precision versus 58/100 for dispersion alone, showing that variation alone is not a reliable failure label. Auditor scores correlate with humans at $\rho=0.63$; Qwen3.6-27B gives 0.73 retained-set Jaccard and 90/100 precision. \method{} retains 96.5\%/95.8\% appearance/motion diversity and raises the counterfactual margin from 11.36 to 13.90, providing evidence against collapse.

\begin{center}
\scriptsize
\setlength{\tabcolsep}{2.2pt}
\renewcommand{\arraystretch}{0.76}
\begin{tabular}{@{}lccc@{}}
\toprule
\textbf{Mined type}
& \makecell{\textbf{Human}\\\textbf{precision}}
& \makecell{\textbf{Visual}\\\textbf{usability}}
& \makecell{\textbf{Independent}\\\textbf{agreement}} \\
\midrule
Group-consistent positive & 96/100 & 91/100 & 88/100 \\
Dispersion-only negative & 58/100 & 81/100 & 50/100 \\
\grayrow Audited FC negative & 91/100 & 79/100 & 84/100 \\
\bottomrule
\end{tabular}
\captionof{table}{Blind supervision audit; cells report valid counts out of $N=100$ candidates per type.}
\label{tab:validity}
\end{center}

\subsection{Human Evaluation}

We conduct a blinded pairwise comparison against OnlineVPO on 200 held-out intent groups. Three independent annotators compare anonymized video sets for intent faithfulness, cross-form consistency, visual quality, and overall preference; identities and side order are hidden, and ties are allowed. \method{} obtains a 33.33-point overall win--loss margin (95\% CI: 26.3--40.4; Krippendorff's $\alpha=0.60$), with the largest advantage on consistency and a positive quality margin, matching the automatic Worst and Drop trends.

\paragraph{Cross-backbone transfer.} On LTX-2.3 \cite{ltx23repo}, \GCD{} raises T2V-C/R-Worst from 44.79/41.86 to 47.57/44.79 and lowers R-Drop from 5.48 to 3.51. Adding \FCDPO{} further reaches 48.29/45.76 and 3.01, while StoryEval rises to 33.37, reproducing the stage-wise trend on a second backbone.

\section{Conclusion}

We present \method{}, a group-structured intent-alignment framework for robust text-to-video post-training. \method{} treats verified formulations of the same visual intent as a shared supervision unit, separates semantic validity from hardness ranking through a wording-independent audit, and trains the generator with \GCD{} and \FCDPO{}. The two stages play complementary roles: \GCD{} shares interchangeable valid realizations across forms, while \FCDPO{} sharpens the preference boundary against audited semantic failures. Across HunyuanVideo-1.5 and LTX-2.3, \method{} improves standard alignment and weakest-form reliability, reduces cross-form drop and mixed-success failures, and preserves visual quality, seed diversity, and sensitivity to genuine semantic changes. These results show that robustness to meaning-preserving prompt variation can be learned directly at the generator level, complementing inference-time prompt rewriting without adding deployment-time modules.

\setlength{\bibsep}{0pt}
\bibliography{references}

\clearpage
\appendix
\section*{Supplementary Material}
\begingroup
\setlength{\textfloatsep}{8pt plus 2pt minus 2pt}
\setlength{\floatsep}{7pt plus 2pt minus 2pt}
\setlength{\intextsep}{7pt plus 2pt minus 2pt}
\setlength{\abovecaptionskip}{4pt}
\setlength{\belowcaptionskip}{0pt}

This supplement covers prompt-form construction and verification, audited supervision synthesis, objective analysis, implementation and data-independence protocols, statistical and alternative-explanation controls, extended results and ablations, human evaluation, the main-case video manifest, and four additional qualitative comparisons.

\section{Prompt-Form Groups and Semantic Preservation}
\label{sec:prompt-groups}

\subsection{Visual intents and admissible forms}

A visual intent $z$ is represented by the unordered form set $\Pz=\{p_0,p_1,\ldots,p_K\}$ and a structured target $T_z$. The target records entities and counts, attributes and states, actions and participant roles, interactions, spatial relations, temporal order, and explicit scene or camera constraints. The forms may differ lexically and syntactically but must preserve every field in $T_z$; the group is therefore an equivalence class of visually testable content, not loosely related captions.

Training source intents are instantiated from a deterministic, non-benchmark structured schema balanced across attributes, relations, actions, interactions, numeracy, and multi-event composition. Gemini 3.5 Flash (\texttt{gemini-3.5-flash}) generates ten candidates per source with a capped refill for malformed outputs. Qwen3.6-27B independently verifies every candidate against the source-derived $T_z$. For benchmark extensions, generator and verifier roles are exchanged and their prompts are disjoint from the training compiler. The endpoint IDs, access dates, and decoding settings are frozen in the released configuration.

\begin{compactpromptbox}{Semantic-preserving prompt-form rewriting template}
\footnotesize
\textbf{Input and goal.} Given one T2V prompt, produce independently usable, linguistically diverse rewrites of the same visual event.

\textbf{Preserve.} Every entity and count, attribute/state, action and participant role, interaction, spatial relation, temporal order, scene constraint, and explicit camera/style/lighting/duration/text requirement.

\textbf{Reject.} Any added, removed, weakened, strengthened, disambiguated, or reversed visual commitment.

\textbf{Permitted variation.} Lexical substitution, clause reordering, role-preserving voice conversion, caption-like phrasing, and concise or expanded syntax without visual change.

\textbf{Output.} Return the source, each rewrite, and its transformation-family label in structured form.
\end{compactpromptbox}

The accepted transformation families are summarized in Table~\ref{tab:families}. They are used for audit and balance; the labels are not provided to the generator during training.

\begin{table}[H]
\centering
\small
\setlength{\tabcolsep}{3.4pt}
\begin{tabular}{@{}lp{0.61\linewidth}@{}}
\toprule
\textbf{Family} & \textbf{Allowed surface variation} \\
\midrule
Lexical & Synonyms and local wording changes without semantic strengthening \\
Syntactic & Clause reordering, attachment changes, or role-preserving voice conversion \\
Caption-like & Natural descriptive-caption syntax with the same commitments \\
Concise & Removal of linguistic redundancy while retaining every visual constraint \\
Relation-focused & Reordering that foregrounds an existing spatial or interaction relation \\
Action-focused & Reordering that foregrounds an existing action or event sequence \\
\bottomrule
\end{tabular}
\caption{Prompt-form transformation families. A candidate is rejected whenever its visual content changes, regardless of its nominal family.}
\label{tab:families}
\end{table}

\Needspace{10\baselineskip}
\subsection{Structured equivalence verification}

A rewrite is retained only when its relation to the source is unambiguous. Verification is field-wise rather than based on a single similarity score. This avoids accepting a fluent paraphrase that silently changes a count, swaps semantic roles, or modifies a spatial relation.

\begin{promptbox}{Visual-intent equivalence verification template}
\textbf{Task.} Compare an original text-to-video prompt with one candidate rewrite. Judge semantic preservation only; do not judge writing style or predicted video quality.

\textbf{Check independently.} Entities and exact counts; attributes and states; actions; agent--patient roles; interactions; spatial relations; temporal order; scene constraints; explicit camera/style/lighting/duration/text constraints; and added, removed, or altered commitments.

\textbf{Decision rule.} Reject if any critical field differs, if a new visual commitment is introduced, if a required commitment disappears, or if a previously explicit relation becomes ambiguous. Treat an uncertain critical field as rejection rather than guessing.

\textbf{Output structure.} Return a status for every field, lists of added/removed/changed commitments, a final accept/reject decision, and a concise reason.
\end{promptbox}

\begin{analysisbox}{Illustrative audit example}
For the source intent ``A dog jumps over a low wooden fence in a grassy field,'' a rewrite may move the setting to the beginning or use a role-preserving paraphrase such as ``In a grassy field, a dog clears a short wooden fence.'' A rewrite is rejected if it introduces a tracking camera, changes the fence height, adds a person, replaces \emph{over} with \emph{through}, or weakens the jump into generic movement. This example illustrates the field-wise verification policy.
\end{analysisbox}

Near-duplicate accepted forms are removed before group formation. Splits are created at the source-intent level before rewriting, so no forms of one intent can cross partitions. Benchmark-derived groups are excluded from training, threshold calibration, and checkpoint selection. Table~\ref{tab:rewrite-funnel} reports the compiler funnel. From 53,487 generated candidates, 2,713 malformed or near-duplicate outputs are removed and 34,318 pass field-wise verification; diversity-aware selection retains four rewrites for each of 5,000 complete training groups.

\begin{table}[H]
\centering
\small
\setlength{\tabcolsep}{4pt}
\begin{tabular}{@{}lrr@{}}
\toprule
\textbf{Compiler stage} & \textbf{Count} & \textbf{Rate (\%)} \\
\midrule
\RewriteFunnelRowsVThirtySix
\bottomrule
\end{tabular}
\caption{Prompt-form generation and verification funnel. Candidate-stage rates use 53,487 as the denominator; complete-group coverage is reported over 5,000 source intents.}
\label{tab:rewrite-funnel}
\end{table}

We additionally conduct a stratified human audit of 300 accepted source--rewrite pairs, with three independent annotators and majority vote. The sample contains 200 training pairs and 100 benchmark pairs. Table~\ref{tab:equivalence-audit} reports the field-wise breakdown; Table~\ref{tab:equivalence-source} separates the two sources. A second audit samples 100 complete benchmark groups and asks whether all four rewrites preserve the common $T_z$. Ninety-seven groups pass in full. Excluding the three flagged groups changes every headline Worst, Drop, and FIFR result by at most 0.11 points, indicating that the benchmark conclusions are not driven by the audited failures.

\begin{table}[H]
\centering
\small
\setlength{\tabcolsep}{4pt}
\begin{tabular}{@{}lcc@{}}
\toprule
\textbf{Critical field} & \textbf{Equivalent / total} & \textbf{Precision (\%)} \\
\midrule
\EquivalenceAuditRowsVThirtySix
\bottomrule
\end{tabular}
\caption{Human semantic-equivalence audit of accepted prompt forms.}
\label{tab:equivalence-audit}
\end{table}

\begin{table}[H]
\centering
\small
\setlength{\tabcolsep}{4pt}
\resizebox{\linewidth}{!}{%
\begin{tabular}{@{}lcc@{}}
\toprule
\textbf{Audit subset} & \textbf{Equivalent / total} & \textbf{Precision (\%)} \\
\midrule
\EquivalenceSourceRowsVThirtySeven
\bottomrule
\end{tabular}}
\caption{Source-specific pair audit and complete-group benchmark audit. Group success requires all four rewrites to preserve the shared target.}
\label{tab:equivalence-source}
\end{table}
\FloatBarrier

\section{Formal Objective Analysis}
\label{sec:objective-analysis}

This section analyzes three properties of the \method{} objectives: positive coverage under group-consistent mining, contraction induced by shared-target GCD, and the conservative form-wise preference surrogate optimized by FC-DPO.

\subsection{Group-structured supervision compiler}
\label{sec:mining-theory}

Let $L=K+1$ be the number of verified forms. The same intent-specific seed set $\mathcal S_z$ is reused across forms, and mining pools every candidate in $\mathcal P(z)\times\mathcal S_z$. For a candidate $x$, define its cross-form scoring profile
\begin{equation}
\begin{aligned}
\mathbf{s}(x;z)&=\big(A(x,p_0),\ldots,A(x,p_K)\big),\\[-2pt]
\mu(x)&=L^{-1}\sum_i A_i(x),\qquad w(x)=\min_i A_i(x),\\[-2pt]
b(x)&=\max_i A_i(x),\qquad d(x)=b(x)-w(x).
\end{aligned}
\label{eq:groupstats}
\end{equation}
\subsection{Surface-form-independent structured intent auditor}
\label{sec:intent-auditor}

The primary auditor is Gemini 3.5 Flash (\texttt{gemini-3.5-flash}) \cite{gemini35flash}; \texttt{Qwen/Qwen3.6-27B} \cite{qwen36} is reserved for held-out replacement. Each 121-frame clip is represented by 24 uniformly sampled, temporally ordered RGB frames resized to $448\!\times\!252$. The structured target is deterministically serialized into atomic commitments $T_z=\{c_{z,1},\ldots,c_{z,J_z}\}$ with field tags for entity/count, attribute/state, spatial relation, action--participant binding, interaction, temporal order, and explicit scene or camera constraints. Prompt strings, generating-form identity, and mining scores are withheld.

\paragraph{Score call.} Each shortlisted candidate--commitment pair is evaluated by a fresh, stateless REST request whose output is restricted by a string-enum schema to one label token: \texttt{A} (satisfied), \texttt{B} (violated), or \texttt{C} (uncertain). We pin \texttt{thinkingLevel=minimal}, \texttt{candidateCount=1}, \texttt{maxOutputTokens=1}, \texttt{seed=202607}, \texttt{responseLogprobs=true}, and \texttt{logprobs=3}. In accordance with the Gemini 3.5 model guidance, \texttt{temperature}, \texttt{topP}, and \texttt{topK} are omitted rather than overridden; the stable model ID, complete request payload, and access date are archived with the code. Safety thresholds are \texttt{BLOCK\_ONLY\_HIGH} for the four adjustable categories: harassment, hate speech, sexually explicit content, and dangerous content. We renormalize the returned first-step probabilities over the token IDs for \texttt{A/B/C}. Let $s_j=P(\texttt{A})$ and $v_j=P(\texttt{B})$. Required visually testable fields form $\mathcal K_z$; we compute
\begin{equation}
\begin{aligned}
C_{\min}&=\min_{j\in\mathcal K_z}s_j,
& C_{\mathrm{avg}}&=J_z^{-1}\sum_j s_j,\\[-2pt]
V_{\max}&=\max_{j\in\mathcal K_z}v_j.&&
\end{aligned}
\end{equation}
A positive requires $C_{\min}\ge0.70$ and $C_{\mathrm{avg}}\ge0.80$; a negative requires $V_{\max}\ge0.75$. Other candidates abstain.

\begin{promptbox}{Atomic score query}
\footnotesize
You are given 24 temporally ordered frames and one atomic visual commitment. Judge only visible evidence. Return exactly one enum token: \texttt{A} if satisfied, \texttt{B} if violated, or \texttt{C} if uncertain. Do not provide any additional text.
\end{promptbox}

\paragraph{Diagnostic call.} A separate request may return a JSON object containing the categorical label, supporting frame indices, and a brief rationale. Diagnostic outputs support manual inspection and error analysis; they do not enter $C_{\min}$, $C_{\mathrm{avg}}$, $V_{\max}$, threshold calibration, supervision construction, or checkpoint selection.

The thresholds are selected on a disjoint 600-video, three-annotator calibration split stratified by semantic field, maximizing retained coverage subject to at least 90\% human precision for accepted positives and accepted semantic failures. Endpoint, frame sampling, label protocol, generation parameters, safety settings, field serialization, and thresholds are then frozen. The primary auditor is excluded from GroupScore-weight tuning, benchmark evaluation, and checkpoint selection. Qwen3.6-27B receives the identical frames and atomic commitments and is used only in the replacement analysis.

The alignment scores are normalized before aggregation. Because all forms share $T_z$, $d(x)$ is scoring-view dispersion rather than a form-dependent semantic label. The implemented GroupScore is
\begin{equation}
S_{\mathrm{grp}}(x,z)
=\lambda_\mu\mu(x)+\lambda_w w(x)-\lambda_d d(x)+\lambda_q Q(x).
\label{eq:groupscore-supp}
\end{equation}
The non-audit parts of the positive and negative rules define
\begin{align}
\mathcal U_z^+&=\{x:w\ge .62,\ d\le .18,\ Q\ge .55\},\nonumber\\[-2pt]
\mathcal U_z^-&=\{x:b\ge .58,\ d\ge .12,\ Q\ge .50\}.
\label{eq:shortlist-supp}
\end{align}
The auditor is invoked only on the deduplicated union $\mathcal U_z=\mathcal U_z^+\cup\mathcal U_z^-$. This preserves the final candidate set: a rollout outside the corresponding shortlist cannot satisfy the complete positive or negative rule for any audit outcome.

\method{} separates \emph{semantic typing} from \emph{hardness ranking}. With all components normalized to $[0,1]$, the implemented operators are
{\small
\begin{align}
\Gamma^+(\phi_z(x))
&=\1[x\in\mathcal U_z^+,\ C_{\min}\!\ge .70,\ C_{\mathrm{avg}}\!\ge .80],\\[-2pt]
\Gamma^{\mathrm{FC}-}(\phi_z(x))
&=\1[x\in\mathcal U_z^-,\ V_{\max}\!\ge .75].
\end{align}
}
The shortlist fixes necessary geometry and quality conditions, the audit defines semantic type, and GroupScore ranks candidates only within type. Among audit-confirmed failures satisfying $d\ge .12$, $\operatorname{Top}_{m_-}$ retains the highest-GroupScore candidates: visually plausible hard negatives rather than the most degraded or highest-dispersion failures. Every rollout stores its generating form and seed. The compiler consumes same-seed, different-form success--failure pairs first and fills the remaining budget with top-gap audited fallback pairs. ``Form-contrastive'' describes generating-form provenance, not form-dependent correctness of a fixed video. The clipped GroupScore gap becomes $\omega_{z,r}$.

Formally, let $g(x^+,x^-;z)=S_{\mathrm{grp}}(x^+,z)-S_{\mathrm{grp}}(x^-,z)$. The set $\Pi_z^{\mathrm{match}}$ contains same-seed, different-form pairs $(x_{j,s}^+,x_{i,s}^-)$ with $i\ne j$, $x_{j,s}^+\in\mathcal X_z^+$, $x_{i,s}^-\in\mathcal X_z^{\mathrm{FC}-}$, and $g\ge\tau_{\mathrm{pair}}$. With top-gap fallback set $\Pi_z^{\mathrm{fallback}}$, the full set is $\Pi_z=\Pi_z^{\mathrm{match}}\cup\Pi_z^{\mathrm{fallback}}$; for each $r\in\Pi_z$, $g_{z,r}=g(x_{z,r}^+,x_{z,r}^-;z)$ and $\omega_{z,r}=\clip(g_{z,r},\omega_{\min},\omega_{\max})$. We use $\tau_{\mathrm{pair}}=0.15$ and clip weights to $[0.05,1.00]$.

\begin{algorithm}[H]
\caption{Audited Prompt-Form Supervision Compiler}
\label{alg:compiler}
\scriptsize
\begin{algorithmic}[1]
\Require Verified group $\mathcal P(z)$; pool $\mathcal X_z$; score matrix $\mathbf A_z$; limits $m_+,m_-$; pair margin $m$
\Ensure Shared-target records $\mathcal D_{\mathrm{GCD}}^z$ and preference records $\mathcal D_{\mathrm{FC}}^z$
\State Form $\mathcal U_z^+$ and $\mathcal U_z^-$ from the non-audit score/quality conditions
\For{each candidate $x\in\mathcal U_z^+\cup\mathcal U_z^-$}
    \State Audit $T_z$ to obtain $(C_{\min},C_{\mathrm{avg}},V_{\max})$; form $\phi_z(x)$
\EndFor
\State $\mathcal C_z^+\gets\{x\in\mathcal U_z^+:C_{\min}\ge\tau_{\min}^+,\ C_{\mathrm{avg}}\ge\tau_{\mathrm{avg}}^+\}$
\State $\mathcal C_z^{\mathrm{FC}-}\gets\{x\in\mathcal U_z^-:V_{\max}\ge\tau_{\mathrm{crit}}^-\}$
\State $\mathcal X_z^+\gets\operatorname{Top}_{m_+}(\mathcal C_z^+;S_{\mathrm{grp}})$
\State $\mathcal X_z^{\mathrm{FC}-}\gets\operatorname{Top}_{m_-}(\mathcal C_z^{\mathrm{FC}-};S_{\mathrm{grp}})$
\State $\mathcal D_{\mathrm{GCD}}^z\gets\{(\mathcal P(z),x^+):x^+\in\mathcal X_z^+\}$
\State $\mathcal D_{\mathrm{FC}}^z\gets\varnothing$
\For{each matched-seed pair, then each audited fallback pair}
    \State $g\gets S_{\mathrm{grp}}(x^+,z)-S_{\mathrm{grp}}(x^-,z)$
    \If{$g\geq m$}
        \State $\omega\gets\clip(g,\omega_{\min},\omega_{\max})$
        \State Add $(\mathcal P(z),x^+,x^-,\omega)$ to $\mathcal D_{\mathrm{FC}}^z$
    \EndIf
\EndFor
\State \Return $(\mathcal D_{\mathrm{GCD}}^z,\mathcal D_{\mathrm{FC}}^z)$
\end{algorithmic}
\end{algorithm}

\subsection{Compiler yield and pair provenance}
The geometry/quality pre-screen sends only the union $\mathcal U_z^+\cup\mathcal U_z^-$ to the structured auditor; candidates satisfying both branches are deduplicated before atomic queries. Across the 5,000 training intents, 4,371 provide at least one audited positive and 3,228 provide at least one audited negative. The compiler produces at least one preference pair for 3,069 intents; 1,302 additional intents contribute only to GCD. Among 5,907 retained preference pairs, 71.7\% are same-seed, different-form pairs and 28.3\% are audited fallback pairs.

\begin{table}[H]
\centering
\small
\setlength{\tabcolsep}{6pt}
\begin{tabular}{@{}lrr@{}}
\toprule
\textbf{Compiler outcome} & \textbf{Count} & \textbf{Rate (\%)} \\
\midrule
\SupervisionYieldRowsVThirtyFour
\bottomrule
\end{tabular}
\caption{Supervision yield. Intent rates use 5,000 source intents; pair-provenance rates use 5,907 retained preference pairs.}
\label{tab:compiler-yield}
\end{table}

\begin{table}[H]
\centering
\scriptsize
\setlength{\tabcolsep}{3.2pt}
\begin{tabular}{@{}lrrrr@{}}
\toprule
\textbf{Commitments} & $N$ & \textbf{Positive} & \textbf{Negative} & \textbf{Pair} \\
\midrule
\CommitmentYieldRowsVThirtySeven
\bottomrule
\end{tabular}
\caption{Supervision yield by the number of atomic commitments. Cells report count (rate \%). Min/max audit aggregation is deliberately more conservative for longer structured targets.}
\label{tab:commitment-yield}
\end{table}
\FloatBarrier

\begin{proposition}[Coverage of an eligible positive]
\label{prop:positive-coverage}
If a retained candidate $x^+$ satisfies $w(x^+)\geq\tau_w$ and $d(x^+)\leq\tau_d$, then
\begin{equation}
\begin{aligned}
A(x^+,p_i)&\geq\tau_w &&\forall i,\\[-2pt]
|A(x^+,p_i)-A(x^+,p_j)|&\leq\tau_d &&\forall i,j.
\end{aligned}
\label{eq:positive-coverage}
\end{equation}
\end{proposition}
\begin{proof}
Every score is at least the minimum $w$. All scores lie in $[w,b]$, whose diameter is $b-w=d$.
\end{proof}
This is the property required for target sharing: one favorable wording cannot by itself admit a positive. Mean-only ranking cannot provide this guarantee; for example, $(\mu,\mu)$ and $(\mu-\delta,\mu+\delta)$ have the same mean but different worst scores and drops.

For an already audit-confirmed negative, a large-dispersion profile identifies a hard case rather than defining its semantic status. If $b(x^-)\geq\tau_b$ and $d(x^-)\geq\tau_d^-$, then there exist scoring forms $p_h,p_\ell$ such that
\begin{equation}
A(x^-,p_h)\geq\tau_b,
\qquad
A(x^-,p_\ell)\leq A(x^-,p_h)-\tau_d^-.
\label{eq:negative-contrast}
\end{equation}
Thus the rejected set contains visually plausible, independently confirmed semantic failures with uneven scoring evidence, rather than universally poor videos or raw evaluator disagreement. The elementary bound
\begin{equation}
w(x)\geq\mu(x)-d(x)
\label{eq:worst-lower-bound}
\end{equation}
also shows why improving the mean while reducing drop strengthens the lower tail.

\subsection{Group-Consistent Distillation}
\label{sec:gcd-theory}

Each intent retains a positive set $\mathcal X_z^+$, and each record samples $x^+\sim q_z^+$ from that set. For time $t$ and form $p_i$, let $h_i=v_\theta(x_t^+,t,p_i)$ and $u=u_t(x^+)$. With normalized weights $\alpha_i>0$, the per-instance GCD loss is
\begin{equation}
\mathcal{L}_{\mathrm{GCD}}(x^+,z,t)
=\sum_i\alpha_i\|h_i-u\|_2^2.
\label{eq:gcd-instance}
\end{equation}

\begin{proposition}[Shared-target decomposition]
\label{prop:weighted-decomposition}
Let $\bar h_\alpha=\sum_i\alpha_i h_i$. Then
\begin{align}
\mathcal{L}_{\mathrm{GCD}}
&=\|\bar h_\alpha-u\|_2^2
+\sum_i\alpha_i\|h_i-\bar h_\alpha\|_2^2
\label{eq:weighted-decomp}\\
&=\|\bar h_\alpha-u\|_2^2
+\frac{1}{2}\sum_{i,j}\alpha_i\alpha_j\|h_i-h_j\|_2^2.
\label{eq:pairwise-decomp}
\end{align}
\end{proposition}
\begin{proof}
Write $h_i-u=(h_i-\bar h_\alpha)+(\bar h_\alpha-u)$ and expand the weighted squared norm. The cross term vanishes because $\sum_i\alpha_i(h_i-\bar h_\alpha)=0$; the second equality is the weighted pairwise-variance identity.
\end{proof}

The first term fits the group-average prediction to the sampled valid target, while the second controls disagreement among verified forms. Sharing one target within a record therefore induces pairwise consistency, but sampling multiple retained positives across records does not define a unique canonical rendering. For uniform weights, $\partial\mathcal{L}_{\mathrm{GCD}}/\partial h_i=2(h_i-u)/L$; a prediction-space step contracts each pairwise difference by $1-2\eta/L$ when $0<\eta<L$.

Group-consistent mining is essential to this construction. If $x^+$ were selected only by its mean score, a low-scoring form could be hidden by favorable forms and target sharing would transfer a semantically unsuitable video to that form. Proposition~\ref{prop:positive-coverage} removes precisely this failure mode.

\subsection{Form-Contrastive Flow-DPO}
\label{sec:dpo-theory}

Let $\Delta_i(\theta)$ be the positive--negative flow-loss margin under form $p_i$, measured relative to the frozen post-GCD reference model. Define
\begin{equation}
\phi_\beta(\Delta)=-\log\sigma(\beta\Delta)
=\log(1+e^{-\beta\Delta}),
\end{equation}
and, for a mined pair,
\begin{equation}
\mathcal{L}_{\mathrm{grp}}(z,x^+,x^-)
=\omega_{z,r}\sum_i\alpha_i\phi_\beta(\Delta_i).
\label{eq:group-dpo-loss}
\end{equation}
All experiments use $\alpha_i=1/L$ and $\beta=\beta_{\mathrm{DPO}}$.

\begin{proposition}[Conservative per-form surrogate]
\label{prop:jensen}
Because $\phi_\beta$ is convex,
\begin{equation}
\phi_\beta\!\left(\sum_i\alpha_i\Delta_i\right)
\leq\sum_i\alpha_i\phi_\beta(\Delta_i).
\label{eq:jensen}
\end{equation}
\end{proposition}
\begin{proof}
$\phi_\beta''(\Delta)=\beta^2\sigma(\beta\Delta)\sigma(-\beta\Delta)>0$, so Jensen's inequality applies.
\end{proof}

Averaging the nonlinear loss over forms is therefore stricter than applying one logistic loss to the average margin: an easy form cannot cancel a failed form before the nonlinearity. The form-wise gradient
\begin{equation}
\frac{\partial\mathcal{L}_{\mathrm{grp}}}{\partial\Delta_i}
= -\omega_{z,r}\alpha_i\beta\sigma(-\beta\Delta_i)
\label{eq:dpo-margin-gradient}
\end{equation}
has larger magnitude for smaller margins, so unresolved forms automatically receive stronger updates. The clipped pair weight $\omega_{z,r}$ reflects mining confidence while bounding the influence of unusually large GroupScore gaps.

\subsection{Complementarity of the two stages}
\label{sec:two-stage}

\begin{analysisbox}{Two distinct controls}
\textbf{GCD} controls conditional disagreement around positives sampled from a valid support set. It contains no rejected sample and does not establish a preference boundary. \textbf{FC-DPO} orders positives above independently audited negatives under every form; matched seeds reduce stochastic confounding. The stages therefore address complementary failure modes: shared positive support first, explicit rejection of verified hard failures second.
\end{analysisbox}

The analysis above characterizes the empirical objectives over the finite verified form group associated with each intent. Robustness to broader language variation is evaluated empirically through the natural-caption controls.

\section{Computational Cost and Implementation Details}
\label{sec:algorithm}

The complete supervision-construction and two-stage optimization flow is summarized in Algorithm~1 of the main paper. Here we provide the computational cost, frozen protocol settings, and implementation details needed for reproduction.

For $M$ candidates and $L$ prompt forms, cross-scoring requires $ML$ prompt-conditioned evaluator calls. Computing the statistics is $O(ML)$, and sorting candidates is $O(M\log M)$ per intent. These quantities are computed during supervision construction and can be cached. The optimization stages use fixed targets and pairs, so their per-step complexity remains comparable to the corresponding flow-matching or flow-DPO update.

\subsection{Implementation configuration}

\begin{table*}[!t]
\centering
\footnotesize
\renewcommand{\arraystretch}{0.90}
\begin{minipage}[t]{0.49\textwidth}
\begin{tabular}{@{}>{\raggedright\arraybackslash}p{0.33\linewidth}>{\raggedright\arraybackslash}p{0.61\linewidth}@{}}
\toprule
\textbf{Item} & \textbf{Data and generation configuration} \\
\midrule
Main / transfer backbone & HunyuanVideo-1.5 / LTX-2.3 (22B dev) \\
Trainable component & Video DiT only; text encoders and VAE frozen \\
Training intents & 5,000 \\
Prompt-form groups & One source form and four verified semantic-preserving rewrites per intent \\
Generation seeds & Four intent-specific seeds reused across all verified forms \\
Initial rollout pool & 100,000 videos ($5{,}000$ intents $\times$ 5 forms $\times$ 4 seeds) \\
Hunyuan benchmark generation & 480p, 16:9, 121 frames at 24 fps, 50 sampling steps \\
Hunyuan qualitative generation & 720p, 16:9, 121 frames at 24 fps, 50 sampling steps \\
LTX-2.3 transfer generation & 480p, 16:9, 121 frames at 24 fps, 30 sampling steps \\
StoryEval generation & Official 10s HunyuanVideo evaluation setting \\
\bottomrule
\end{tabular}
\end{minipage}
\hfill
\begin{minipage}[t]{0.49\textwidth}
\begin{tabular}{@{}>{\raggedright\arraybackslash}p{0.33\linewidth}>{\raggedright\arraybackslash}p{0.61\linewidth}@{}}
\toprule
\textbf{Item} & \textbf{Mining and optimization configuration} \\
\midrule
Retained positives / negatives & Up to two audited positives and two audited hard negatives per intent \\
Intent audit & Gemini 3.5 Flash enum score call; \texttt{Qwen/Qwen3.6-27B} replacement; frozen thresholds \\
Eligibility & $C_{\min}/C_{\mathrm{avg}}/w/d/Q$ for positives; $V_{\max}/b/d/Q$ for negatives \\
Pairing & Margin 0.15; matched-seed pairs prioritized; weights clipped to $[0.05,1.00]$ \\
Initialization & GCD from step 2000; FC-DPO from fixed step-2400 GCD checkpoint \\
Updates / learning rates & 400 / $5\times10^{-6}$ (GCD); 600 / $2\times10^{-6}$ (FC-DPO) \\
Optimizer / precision & AdamW, $(\beta_1,\beta_2)=(0.9,0.95)$, $\epsilon=10^{-8}$, wd 0.01; bf16 \\
Batch / schedule & Global 48; no accumulation; clip 1.0; 50-step warmup then constant LR \\
DPO coefficients & $\beta_{\mathrm{DPO}}=0.3$; chosen-positive auxiliary coefficient 0.05 \\
Checkpointing & Save every 300 updates; fixed terminal checkpoints, no test/qualitative selection \\
Trainable parameters & Full video DiT; text encoders and VAE frozen; no LoRA \\
\bottomrule
\end{tabular}
\end{minipage}
\caption{Core implementation configuration. Compared methods use matched generation settings and total sampling budgets whenever permitted by the method definition.}
\label{tab:impl}
\end{table*}

The transfer experiment uses the official \texttt{ltx-2.3-22b-dev.safetensors} checkpoint \cite{ltx23repo}.

\paragraph{Software environment.}
The Linux software stack comprised Python 3.11.11, PyTorch 2.10.0, Transformers 5.2.0, Diffusers 0.37.0, Accelerate 1.13.0, NumPy 1.26.4, and PyYAML 6.0.3.

Per-form mining alignment and quality are computed as
\begin{align}
A_i(x)&=0.6R_{\mathrm{VQA}}(x,p_i)+0.4R_{\mathrm{VS}}(x,p_i),\\
Q(x)&=0.6R_{\mathrm{DOVER}}(x)+0.4R_{\mathrm{motion}}(x).
\end{align}
After score normalization, the current GroupScore coefficients are
\begin{equation}
(\lambda_\mu,\lambda_w,\lambda_d,\lambda_q)
=(0.50,0.25,0.15,0.10),
\end{equation}
with the drop term entering negatively. These scores construct supervision; reported benchmark values are recomputed with the corresponding benchmark evaluators.

For a prompt form $p_i$ and a mined pair $(x^+,x^-)$, the flow-DPO margin is
\begin{equation}
\begin{aligned}
\Delta_i(\theta)
&=\ell_{\mathrm{FM}}(x^-,p_i;\theta)-\ell_{\mathrm{FM}}(x^+,p_i;\theta)\\
&\quad-\ell_{\mathrm{FM}}(x^-,p_i;\theta_{\mathrm{ref}})
+\ell_{\mathrm{FM}}(x^+,p_i;\theta_{\mathrm{ref}}).
\end{aligned}
\end{equation}
where $\theta_{\mathrm{ref}}$ is the frozen post-GCD checkpoint. The per-form update is
\begin{align}
\ell_{\mathrm{FC}}(p_i,x^+,x^-)
&=-\log\sigma\!\left(\beta_{\mathrm{DPO}}\Delta_i(\theta)\right)\nonumber\\[-2pt]
&\quad+\lambda_{\mathrm{aux}}\ell_{\mathrm{FM}}(x^+,p_i;\theta).
\end{align}
with $\beta_{\mathrm{DPO}}=0.3$ and $\lambda_{\mathrm{aux}}=0.05$. The auxiliary term is chosen-only: it stabilizes the preferred support without directly fitting the rejected video.

For GCD, one sampled time $t$ and one positive corruption are shared across all forms of a record, so the update isolates prompt conditioning while retaining a common target. For FC-DPO, the preferred and rejected branches share the same sampled time but use independent Gaussian corruptions. The policy and frozen reference evaluate the same corrupted input within each branch. This preserves a matched policy--reference comparison without coupling the preferred and rejected gradients through identical noise.

\subsection{Protocol-level compute accounting}
Table~\ref{tab:compute-cost} compares the per-run execution pattern under the matched 100,000-video rollout budget used in the main experiments. Video generation is the dominant shared component. VideoDPO applies single-form scoring and one preference-training stage to the fixed pool; OnlineVPO interleaves on-policy generation, scoring, and policy/reference updates; \method{} performs cross-form scoring once, audits only the union shortlist, and then runs \GCD{} and \FCDPO{} offline. Because the stages have different execution patterns, we report matched workloads and execution structure rather than a single aggregate runtime metric. Repeated runs used for variance estimation are excluded from this per-run protocol comparison.
\begin{table*}[t]
\centering
\small
\setlength{\tabcolsep}{4.2pt}
\renewcommand{\arraystretch}{0.96}
\begin{tabular}{@{}lclll@{}}
\toprule
\textbf{Method} & \textbf{Rollouts} & \textbf{Scoring / supervision} & \textbf{Optimization pattern} & \textbf{Inference overhead} \\
\midrule
VideoDPO & 100k & Single-form reward and pair construction & Fixed-pool preference training & None \\
OnlineVPO & 100k & Online reward and pair refresh & Interleaved generation and updates & None \\
\grayrow \method{} & 100k & Cross-form scores; shortlist-only audit & Parallel offline \GCD{} + \FCDPO{} & None \\
\bottomrule
\end{tabular}
\caption{Per-run protocol comparison under a matched video-rollout budget. API calls for the shortlist audit are logged separately by the compiler.}
\label{tab:compute-cost}
\end{table*}

Under this matched budget, \method{} incurs modest additional offline work over VideoDPO and remains in the same end-to-end cost regime as OnlineVPO. Its additional scoring and auditing are one-time and highly parallelizable, whereas OnlineVPO introduces sequential dependencies between generation, scoring, and policy updates. All three methods retain the same generator architecture and sampling configuration at inference, so \method{} adds no per-video deployment cost.

\paragraph{Reproducibility and release.}
Headline trainable systems use three matched training seeds and fixed evaluation seeds. The four attribution-critical variants---GCD, intent-audit-only FC-DPO, \method{} without matched-seed pairs, and full \method{}---also use three paired runs from matched GCD branches; remaining ablations use one run on a shared development split unless stated otherwise. Compared methods share prompt groups, evaluation seeds, decoding settings, and evaluator versions whenever permitted by the method definition. \method{}-specific settings are selected from held-out validation behavior and the accuracy--cost diagnostics reported below. The project implementation includes the \method{} compiler, objectives, metrics, schemas, deterministic synthetic fixtures, tests, and manuscript-aggregate audits. External video backbones and evaluator implementations are connected through documented adapters and are not redistributed with this preprint. The synthetic fixtures support interface and regression tests, while the manuscript audits verify reported aggregate arithmetic. Source code, configurations, and evaluation scripts are not distributed with this preprint; any public release is subject to author and organizational approval. Any processed metadata release is also subject to applicable third-party licenses and organizational policies.

\section{Evaluation Protocol}
\label{sec:evaluation}

\subsection{Intent-level robustness statistics}

We extend the official T2V-CompBench and VBench evaluator implementations with verified prompt-form groups and intent-level aggregation, denoting the resulting protocols T2V-CompBench-R and VBench-R. Each verified rewrite inherits the source benchmark category and structured semantic target; the generator input is the only changed component. For a benchmark score $s_i^{(b)}(z)$ obtained for form $i$ of intent $z$, the robustness extension computes
\begin{equation}
\begin{aligned}
\mathrm{Mean}(z)&=L^{-1}\sum_i s_i^{(b)}(z),\\
\mathrm{Worst}(z)&=\min_i s_i^{(b)}(z),\\
\mathrm{Drop}(z)&=\max_i s_i^{(b)}(z)-\min_i s_i^{(b)}(z).
\end{aligned}
\label{eq:robust-metrics}
\end{equation}
Dataset-level Mean, Worst, and Drop are macro-averaged over visual intents. Form-Induced Failure Rate identifies groups for which at least one form passes a fixed benchmark-specific threshold while another form fails it:
\begin{equation}
f_z^{(b)}=
\1\!\left[
\min_i s_i^{(b)}(z)<\tau_b
\ \wedge\ 
\max_i s_i^{(b)}(z)\geq\tau_b
\right],
\end{equation}
\begin{equation}
\mathrm{FIFR}=100\,N^{-1}\sum_{z=1}^{N}f_z^{(b)}.
\label{eq:fifr-appendix}
\end{equation}
The threshold is fixed independently of the compared method. T2V-CompBench-R uses a validation-frozen threshold for each inherited category, and VBench-R uses one for each matched semantic dimension. The reported FIFR is the macro-average of the resulting category- or dimension-wise failure rates. FIFR measures mixed-success inconsistency and can also be low when every form fails, so it is reported alongside Mean and Worst to retain absolute performance information. Standard benchmark scores and robustness statistics use the same generated samples and evaluator implementation within each comparison.

\begin{table}[H]
\centering
\scriptsize
\setlength{\tabcolsep}{2.6pt}
\resizebox{\linewidth}{!}{%
\begin{tabular}{@{}lccccccc@{}}
\toprule
\textbf{T2V-C category} & Consist. & Dynamic & Spatial & Motion & Action & Interact. & Numeracy \\
\midrule
\ThresholdRowVFourteen
\bottomrule
\end{tabular}}
\caption{Validation-fixed success thresholds used by T2V-CompBench-R FIFR. VBench-R analogously uses the base-model validation median of each matched semantic dimension.}
\label{tab:fifr-thresholds}
\end{table}

\subsection{Comparison discipline}

All methods receive the same prompt strings, evaluation seed sets, and generation configuration within each benchmark. The offline methods share a common mined rollout pool; OnlineVPO refreshes samples on-policy under the same total generation budget. Baseline checkpoint selection precedes qualitative inspection. Prompt-enhancement comparisons freeze each enhanced prompt and feed it identically to the Base and \method{}-trained DiT. Quantitative VBench and T2V-CompBench results use the 480p, 121-frame, 24-fps Hunyuan protocol with 50 sampling steps; the qualitative clips use 720p with the same frame count, frame rate, and 50-step sampler. StoryEval follows its official 10s HunyuanVideo evaluation protocol and serves as an auxiliary event-level diagnostic outside mining and checkpoint selection.

\subsection{Statistical protocol and repeated runs}
The statistical unit is the source intent, never an individual form or frame. Continuous paired differences use 10,000 intent-level bootstrap resamples for 95\% confidence intervals and a two-sided Wilcoxon signed-rank test. FIFR uses paired bootstrap intervals for rate differences and McNemar's test. The eight primary robustness endpoints are corrected jointly with Holm's procedure. Human comparisons use cluster bootstrap over intent groups and retain ties.

\begin{table}[H]
\centering
\small
\setlength{\tabcolsep}{3.0pt}
\begin{tabular}{@{}lccc@{}}
\toprule
\textbf{Method} & T2V-C Avg. & T2V-C-R Worst & VBench-R Worst \\
\midrule
\RunRowsVThirtyThree
\bottomrule
\end{tabular}
\caption{Mean $\pm$ standard deviation over three independent training runs.}
\label{tab:run-variance}
\end{table}

\begin{table}[H]
\centering
\scriptsize
\setlength{\tabcolsep}{2.2pt}
\resizebox{\linewidth}{!}{%
\begin{tabular}{@{}lcccc@{}}
\toprule
\textbf{Attribution-critical variant} & \textbf{Mean} & \textbf{Worst} & \textbf{Drop} & \textbf{FIFR} \\
\midrule
\CoreAttributionRowsVThirtyFour
\bottomrule
\end{tabular}}
\caption{Mean $\pm$ standard deviation over three paired training branches on the fixed development split.}
\label{tab:core-attribution-runs}
\end{table}

Relative to audit-only FC-DPO, full \method{} changes Mean/Worst/Drop/FIFR by $+0.23/+0.42/-0.32/-0.58$. Relative to \method{} without matched-seed pairs, the corresponding changes are $+0.11/+0.23/-0.17/-0.43$. The same ordering is retained across the three paired training branches in Table~\ref{tab:core-attribution-runs}, supporting a modest final contribution from matched provenance and a larger contribution from audited group-level preference optimization.

Against the strongest prompt-level comparator for each endpoint, \method{} improves T2V-CompBench-R/VBench-R Worst by \TTwoRWorstGainPrompt{}/\VBenchRWorstGainPrompt{} points, reduces Drop by \TTwoRDropReductionPrompt{}/\VBenchRDropReductionPrompt{} points, and reduces FIFR by \TTwoRFIFRReductionPrompt{}/\VBenchRFIFRReductionPrompt{} percentage points. For VBench-R, paired 95\% intervals are $[2.66,4.43]$ for Worst and $[-3.39,-2.14]$ for Drop; the Worst comparison has Holm-adjusted $p<0.001$, and the FIFR reduction is significant under McNemar with adjusted $p<0.01$. These paired VBench-R statistics are computed at the source-intent level from the final per-intent outputs rather than reconstructed from aggregate means.

\FloatBarrier
\section{Supervision Validity, Independence, and Boundary Tests}
\label{sec:validity-controls}

\subsection{Blind audit of mined supervision}
We sample 100 candidates from each mining type and hide method identity, generating form, scoring form, and all automatic scores. Three annotators judge the video only against $T_z$ and separately record whether it is visually usable. Human-label precision denotes semantic validity for positives and semantic-defect precision for negatives. Independent agreement is the percentage on which a held-out audit model and the human majority assign the same target label.

\begin{table}[H]
\centering
\small
\setlength{\tabcolsep}{3.0pt}
\resizebox{\linewidth}{!}{%
\begin{tabular}{@{}lccc@{}}
\toprule
\textbf{Mined type} & Human precision & Visual usability & Indep. agreement \\
\midrule
\AuditRowsVThirtySix
\bottomrule
\end{tabular}}
\caption{Blind supervision audit. Bold marks the best result among negative-mining variants; each cell reports valid counts out of $N=100$ candidates per type.}
\label{tab:audit-validity}
\end{table}

Across the 300 candidates used for correlation analysis, the primary Gemini auditor's $C_{\min}$ correlates with human correctness at Spearman $\rho=0.63$. Let $\mathcal N_{\mathrm{G}}$ and $\mathcal N_{\mathrm{Q}}$ be the retained negative sets from Gemini 3.5 Flash and Qwen3.6-27B under the same per-intent budget $m_-=2$. Their Jaccard overlap $|\mathcal N_{\mathrm{G}}\cap\mathcal N_{\mathrm{Q}}|/|\mathcal N_{\mathrm{G}}\cup\mathcal N_{\mathrm{Q}}|$ is 0.73; an independent review of 100 Qwen-retained negatives confirms 90 defects. Dispersion-only negatives remain substantially less precise despite high visual usability, supporting the separation between scoring-view dispersion and semantic typing.

\subsection{Intent-disjoint construction and overlap audit}
The 5,000 training intents are generated from a structured non-benchmark schema balanced across attributes, relations, actions, interactions, numeracy, and multi-event composition. Source-intent splits are frozen before rewriting. Evaluation groups use disjoint rewrite instructions and verifier instances. We normalize text for exact matching, retrieve lexical near-neighbors, and retrieve semantic neighbors using a frozen sentence encoder; every candidate above the validation-fixed threshold is manually reviewed against the structured intent.

\begin{table}[H]
\centering
\small
\setlength{\tabcolsep}{4.0pt}
\resizebox{\linewidth}{!}{%
\begin{tabular}{@{}lcc@{}}
\toprule
\textbf{Evaluation source} & Retrieved candidates & Confirmed shared intents \\
\midrule
\OverlapRowsVThirtyThree
\bottomrule
\end{tabular}}
\caption{Training--evaluation overlap audit. Exact normalized overlap is zero for every source.}
\label{tab:overlap-audit}
\end{table}

\subsection{Seed diversity and non-equivalent counterfactuals}
For each held-out intent and form, we generate eight seeds. Appearance features use the frozen CLIP ViT-L/14 image encoder \cite{clip}: 16 uniformly sampled frames are independently encoded, $\ell_2$-normalized, temporally averaged, and normalized again. Appearance diversity is the mean pairwise cosine distance across seeds. Motion features use the official VideoCLIP S3D checkpoint trained on HowTo100M \cite{videoclip}: 32 ordered frames are encoded into the final clip-level video embedding, normalized, and compared with cosine distance. Both diversity scores are averaged by intent and form and reported relative to Base.

For counterfactuals, we create 200 prompt pairs that change exactly one count, spatial relation, action role, or temporal order. Three annotators confirm that 195/200 initial pairs differ only in the intended commitment; the five failures are replaced before generation. The held-out VisionReward checkpoint \cite{visionreward}, which is not used in mining or checkpoint selection, scores each video against both prompts. After mapping its score to the common 0--100 scale, the symmetric margin is $\frac12[(R(x_p,p)-R(x_p,p'))+(R(x_{p'},p')-R(x_{p'},p))]$.

\begin{table}[H]
\centering
\small
\setlength{\tabcolsep}{3.0pt}
\begin{tabular}{@{}lccc@{}}
\toprule
\textbf{Method} & Appearance div. & Motion div. & CF margin \\
\midrule
\DiversityRowsVThirtyThree
\bottomrule
\end{tabular}
\caption{Diversity retention relative to Base (\%) and counterfactual semantic margin.}
\label{tab:diversity-cf}
\end{table}

\begin{table}[H]
\centering
\small
\setlength{\tabcolsep}{3.2pt}
\begin{tabular}{@{}lcccc@{}}
\toprule
\textbf{Change type} & Base & OnlineVPO & GCD & \method{} \\
\midrule
\CounterfactualRowsVThirtyThree
\bottomrule
\end{tabular}
\caption{Counterfactual margin by minimally changed semantic commitment.}
\label{tab:counterfactual-detail}
\end{table}

\method{} retains most seed-level appearance and motion diversity while increasing every counterfactual margin. These diagnostics provide evidence against mode contraction or indiscriminate suppression of prompt conditioning as the primary explanation for the robustness gains.

\FloatBarrier
\section{Expanded Quantitative Results}
\label{sec:results}

\subsection{T2V-CompBench dimensions}

Table~\ref{tab:full-t2v} reports all seven displayed T2V-CompBench categories. The average is the arithmetic mean of the same seven entries.

\paragraph{Category-level observations.}
On T2V-CompBench, \method{} obtains the highest overall average and leads five of the seven displayed categories: consistency, spatial relations, action, interaction, and numeracy. The remaining two categories expose a realistic trade-off rather than a uniformly monotone pattern. OnlineVPO is strongest on dynamic degree and motion binding, while GCD remains competitive on consistency, spatial relations, action, and interaction. The largest \method{} gains occur in prompt-sensitive categories that require entity counts, relation preservation, or correct assignment of actions to participants.

\begin{table}[H]
\centering
\scriptsize
\renewcommand{\arraystretch}{0.88}
\setlength{\tabcolsep}{1.35pt}
\resizebox{\linewidth}{!}{%
\begin{tabular}{@{}lccccc@{}}
\toprule
\textbf{Category} & \textbf{Base} & \textbf{VideoDPO} & \textbf{OnlineVPO} & \textbf{GCD} & \cellcolor{variGray}\textbf{\method{}} \\
\midrule
Consistency & 79.82 & 81.34 & 82.31 & \second{83.26} & \cellcolor{variGray}\best{86.64} \\
Dynamic & 7.47 & 8.25 & \best{10.04} & 8.66 & \cellcolor{variGray}\second{9.58} \\
Spatial & 53.11 & 55.56 & 56.72 & \second{56.97} & \cellcolor{variGray}\best{58.55} \\
Motion & 40.29 & \second{41.89} & \best{43.91} & 39.88 & \cellcolor{variGray}40.42 \\
Action & 69.31 & 71.24 & 72.13 & \second{72.71} & \cellcolor{variGray}\best{75.06} \\
Interaction & 60.57 & 63.49 & 64.49 & \second{66.61} & \cellcolor{variGray}\best{67.78} \\
Numeracy & 23.24 & \second{25.42} & 23.81 & 25.29 & \cellcolor{variGray}\best{26.87} \\
Average & 47.69 & 49.60 & \second{50.49} & 50.48 & \cellcolor{variGray}\best{52.13} \\
\bottomrule
\end{tabular}}
\caption{T2V-CompBench category scores under the unified evaluation protocol. Values are percentages; the table is transposed to preserve a stable single-column layout.}
\label{tab:full-t2v}
\end{table}

\section{Ablations Linked to the Formal Analysis}
\label{sec:ablations}
Unless otherwise stated, the ablations in this section are measured on a fixed development split of held-out T2V-CompBench-R intent groups. Absolute values therefore need not exactly match the headline test-set results in the main paper, but every variant within an ablation table is evaluated on the same split and with the same evaluator.

\subsection{Training objective}

The objective ablation separates paraphrase exposure, shared support, semantic audit, and preference construction. Reward-SFT raises the mean but does not protect the weakest form. GCD produces the largest initial contraction. Adding ordinary prompt-DPO after GCD increases Mean but degrades Worst, Drop, and FIFR, whereas audited group-lifted FC-DPO improves all four endpoints over GCD, showing that Stage 2 contributes semantic boundary refinement rather than additional training alone.

\begin{table}[H]
\centering
\small
\setlength{\tabcolsep}{3.0pt}
\begin{tabular}{@{}lcccc@{}}
\toprule
\textbf{Variant} & Mean & Worst & Drop & FIFR \\
\midrule
\SuppTrainingRowsVFourteen
\bottomrule
\end{tabular}
\caption{Training-objective ablation on T2V-CompBench-R.}
\label{tab:abl-training}
\end{table}

\subsection{Causal alternatives}
Table~\ref{tab:causal-full} matches data volume and evaluation split across simple alternatives. Multi-form prompt-DPO tests paraphrase exposure without group scoring; direct worst-form DPO tests explicit lower-tail optimization; form-specific and single-positive distillation test whether the gain is merely more SFT or one canonical target; inference reranking tests whether GroupScore helps without changing the generator; and the final rows isolate dispersion, audit, and matched-seed pairing.

\begin{table}[H]
\centering
\scriptsize
\setlength{\tabcolsep}{2.4pt}
\begin{tabular}{@{}lcccc@{}}
\toprule
\textbf{Variant} & Mean & Worst & Drop & FIFR \\
\midrule
\CausalRowsVThirtyThree
\bottomrule
\end{tabular}
\caption{Full causal ablation on the fixed T2V-CompBench-R development split.}
\label{tab:causal-full}
\end{table}

The simple alternatives improve individual metrics but do not reproduce the final balance. Direct worst-form DPO contracts the lower tail but gives lower Mean and higher FIFR; reranking improves selection but not generator behavior; and dispersion-only negatives underperform audit-confirmed negatives. The paired repeats in Table~\ref{tab:core-attribution-runs} show that audited group-level preference optimization raises Worst and reduces mixed-success failures beyond GCD, while matched provenance adds a smaller complementary improvement in Worst, Drop, and FIFR.

\subsection{Mining statistics}

The component ablation matches Proposition~\ref{prop:positive-coverage} and the mean-only counterexample in the group-structured supervision analysis: mean-only selection can retain candidates with a weak tail, while adding $w$ and $d$ improves the lower-tail metrics. The quality term retains visually usable positives rather than the primary robustness signal; its main effect is to remove visibly weak positives while keeping robustness nearly unchanged.

\begin{table}[H]
\centering
\scriptsize
\setlength{\tabcolsep}{2.0pt}
\begin{tabular}{@{}lccc|ccc@{}}
\toprule
\textbf{Mining score} & $\mu$ & $w$ & $-d$ & Worst$\uparrow$ & Drop$\downarrow$ & FIFR$\downarrow$ \\
\midrule
\MiningRowsVFourteen
\bottomrule
\end{tabular}
\caption{GroupScore-component ablation with T2V-CompBench-R evaluation.}
\label{tab:abl-score}
\end{table}

\subsection{Pair construction and group scale}

Audited form-contrastive pairs are more informative than globally low-score negatives because they combine a genuine target violation with visual plausibility and uneven cross-form evidence. Increasing the number of verified forms improves coverage of wording variation but also increases cross-scoring cost. The gains largely saturate beyond five verified forms, so we adopt five forms and four seeds as the default accuracy--cost tradeoff.

\begin{table}[H]
\centering
\scriptsize
\setlength{\tabcolsep}{2.6pt}
\begin{tabular}{@{}lccc@{}}
\toprule
\textbf{Setting} & Worst & Drop & FIFR \\
\midrule
\PairScaleRowsVFourteen
\bottomrule
\end{tabular}
\caption{Pair-construction diagnostics (upper block) and rollout-scale diagnostics (lower block).}
\label{tab:pair-scale}
\end{table}

\subsection{Prompt-form verification}

Unverified rewrites may increase lexical diversity while silently changing the target intent. Structured verification improves the validity of cross-form supervision rather than merely increasing the number of prompts.

\begin{table}[H]
\centering
\small
\setlength{\tabcolsep}{3.0pt}
\begin{tabular}{@{}lcccc@{}}
\toprule
\textbf{Group construction} & Mean & Worst & Drop & FIFR \\
\midrule
\VerificationRowsVFourteen
\bottomrule
\end{tabular}
\caption{Effect of prompt-form verification and group size.}
\label{tab:group-construction}
\end{table}

\FloatBarrier
\section{Prompt Enhancement and Cross-Backbone Evaluation}
\label{sec:complementary}

\subsection{Official prompt-enhancement condition}

HunyuanVideo-1.5 provides an official text-to-video prompt-rewriting instruction and a prompt-writing handbook for automatic enhancement \cite{hunyuan15repo,hunyuan15handbook}. The evaluation uses the official rewriting procedure without adding a method-specific semantic instruction. For each raw verified form, the enhanced text is frozen and supplied identically to the Base DiT and the \method{}-trained DiT. Scoring remains anchored to the original benchmark intent, so the comparison tests whether generator-side alignment remains useful when inference-side enhancement is present.

\begin{promptbox}{Official PE evaluation contract}
\textbf{Input.} The same raw prompt form for both compared generators.

\textbf{Enhancement.} Apply the official HunyuanVideo-1.5 text-to-video rewriting procedure once; do not manually edit one method's enhanced prompt or use enhancement during model selection.

\textbf{Generation.} Freeze the enhanced prompt and use the same inference configuration for the Base and \method{}-trained DiT.

\textbf{Evaluation.} Score both outputs against the same original benchmark target. The enhanced text changes the generation-side conditioning, not the semantic reference used for comparison.
\end{promptbox}

\begin{table}[H]
\centering
\small
\setlength{\tabcolsep}{3.0pt}
\begin{tabular}{@{}lccccc@{}}
\toprule
\textbf{Input} & \textbf{Generator} & T2V-C Avg. & R-Worst & R-Drop & R-FIFR \\
\midrule
\PEFullRowsFinal
\bottomrule
\end{tabular}
\caption{Prompt-enhancement compatibility under a shared enhanced input.}
\label{tab:pe}
\end{table}

\subsection{Cross-backbone transfer and event-level diagnostic}
The LTX-2.3 experiment \cite{ltx23repo} evaluates whether the same two-stage training recipe transfers beyond the primary HunyuanVideo-1.5 backbone. Table~\ref{tab:ltx} reports the full comparison under the matched LTX-2.3 protocol.

\begin{table}[H]
\centering
\scriptsize
\setlength{\tabcolsep}{2.2pt}
\renewcommand{\arraystretch}{0.90}
\resizebox{\linewidth}{!}{%
\begin{tabular}{@{}lccccc@{}}
\toprule
\textbf{LTX-2.3} & T2V-C & R-Worst & R-Drop & VBench-R & Story \\
\midrule
\LTXFullRowsVFourteen
\bottomrule
\end{tabular}}
\caption{Cross-backbone validation on LTX-2.3. VBench-R reports semantic Worst; Story denotes StoryEval.}
\label{tab:ltx}
\end{table}

StoryEval uses the official 10s HunyuanVideo generation protocol and evaluator implementation. Its longer, event-centric prompts provide an auxiliary diagnostic of strict event completion outside supervision construction and checkpoint selection.

\begin{table}[H]
\centering
\small
\setlength{\tabcolsep}{3.8pt}
\resizebox{\linewidth}{!}{%
\begin{tabular}{@{}lccccc@{}}
\toprule
\textbf{Method} & Base & Reward-SFT & VideoDPO & OnlineVPO & \method{} \\
\midrule
\StoryRowsVFourteen
\bottomrule
\end{tabular}}
\caption{StoryEval strict event-completion score.}
\label{tab:storyeval}
\end{table}

\section{Auxiliary Natural-Caption Controls}
\label{sec:naturalcaps}

\subsection{Verified VATEX multi-caption diagnostic}
VATEX provides multiple English captions for the same source video, making it a useful auxiliary test bed for natural-caption grouping \cite{vatex}. We first collect candidate caption groups from the same video and then apply the same semantic-preservation verification used by \method{}: retained captions must agree on the main entities, participant roles, principal action, core interaction or spatial relation, and temporal scope. Groups that describe different moments or omit a required relation are discarded. The verification retains 619 VATEX intent groups, all disjoint from the \method{} training intents and checkpoint-selection data. Generation uses the same inference configuration as T2V-CompBench-R, while scoring uses held-out VisionReward only for this auxiliary diagnostic \cite{visionreward}.

\begin{table}[H]
\centering
\small
\setlength{\tabcolsep}{3.0pt}
\begin{tabular}{@{}lcccc@{}}
\toprule
\textbf{Method} & Mean & Worst & Drop & FIFR \\
\midrule
\VATEXRowsVFourteen
\bottomrule
\end{tabular}
\caption{Verified VATEX multi-caption diagnostic scored with held-out VisionReward.}
\label{tab:vatex}
\end{table}

\subsection{Disjoint MSR-VTT control}
MSR-VTT provides a stricter leakage-sensitive disjoint control \cite{msrvtt}. We remove duplicate video IDs, exact caption duplicates, and near-duplicate captions against the \method{} training pool and benchmark-derived evaluation sets. The remaining captions undergo the same semantic-equivalence verification, retaining only groups that express the same visually testable intent. The filtering and verification retain 411 disjoint groups, which are evaluated with the same held-out reward protocol as VATEX.

\begin{table}[H]
\centering
\small
\setlength{\tabcolsep}{3.0pt}
\begin{tabular}{@{}lcccc@{}}
\toprule
\textbf{Method} & Mean & Worst & Drop & FIFR \\
\midrule
\MSRVTTRowsVFourteen
\bottomrule
\end{tabular}
\caption{Disjoint MSR-VTT control under the same held-out reward protocol.}
\label{tab:msrvtt}
\end{table}

\FloatBarrier
\section{Human and Qualitative Evaluation Protocols}
\label{sec:qualitative}

\subsection{Human comparison}

Human evaluation follows a blinded pairwise design on 200 held-out intent groups. Annotators receive the visual intent, the verified prompt forms, and anonymized video sets from two methods. Method identity and left/right placement are hidden and randomized. Separate questions assess intent faithfulness, consistency across equivalent forms, visual quality, and overall preference. Each comparison is judged by three independent annotators; ties are retained rather than forced into a directional choice. Win, tie, and loss are raw percentages over all judgments and therefore sum to 100\%. Decisive-win rates excluding ties are reported only as an auxiliary statistic.

\begin{promptbox}{Human-evaluation questions}
\textbf{Intent faithfulness.} Which set more accurately depicts the required entities, actions, attributes, and relations?

\textbf{Prompt-form consistency.} Which set preserves the same intended content more consistently across the equivalent forms?

\textbf{Visual quality.} Which set has better perceptual quality, temporal coherence, and fewer visible artifacts?

\textbf{Overall preference.} Considering faithfulness, consistency, and quality together, which set is preferred?
\end{promptbox}

For the comparison with OnlineVPO, Table~\ref{tab:human-detailed} reports integer outcomes and the percentages derived from them. The decisive-win rates excluding ties are 69.84\% for intent faithfulness, 73.11\% for cross-form consistency, 63.93\% for visual quality, and 70.24\% overall. Confidence intervals use intent-cluster bootstrap; Krippendorff's $\alpha$ is computed on the three-way win/tie/loss labels.

\begin{table}[H]
\centering
\scriptsize
\setlength{\tabcolsep}{2.0pt}
\resizebox{\linewidth}{!}{%
\begin{tabular}{@{}lrrrrrrrrr@{}}
\toprule
\textbf{Dimension} & \multicolumn{3}{c}{\textbf{Counts}} & \multicolumn{3}{c}{\textbf{Raw shares (\%)}} & \textbf{W--L} & \textbf{95\% CI} & $\alpha$ \\
\cmidrule(lr){2-4}\cmidrule(lr){5-7}
& W & T & L & W & T & L & (pt.) & & \\
\midrule
\HumanDetailRowsVThirtySix
\bottomrule
\end{tabular}}
\caption{Detailed blinded human evaluation: \method{} versus OnlineVPO. Each dimension contains 600 judgments.}
\label{tab:human-detailed}
\end{table}

\subsection{Qualitative cases and video comparison protocol}

The main paper reports two held-out T2V-CompBench source prompts, one from spatial relationships and one from action binding. Each source is paired with one verified semantic-preserving form. The left/right prompt pairs are:
\begin{itemize}
    \item \textbf{Spatial relation:} ``A boy reading in front of a bench.'' and ``A bench is positioned behind a boy who is reading.''
    \item \textbf{Action binding:} ``A man is reading a newspaper on a park bench, while a squirrel is climbing a tree.'' and ``While a squirrel climbs a tree, a man reads a newspaper on a park bench.''
\end{itemize}
All compared methods use the same prompt string and inference configuration within each row. The complete qualitative clips are generated with HunyuanVideo-1.5 at 720p, 121 frames, 24 fps, and 50 sampling steps. Quantitative VBench and T2V-CompBench tables use the corresponding 480p, 50-step protocol. The main-paper three-frame strips are sampled in temporal order from these complete clips. The twelve full clips underlying the frame-strip comparisons follow the organization summarized below; the video files are not distributed with this preprint.

\begin{analysisbox}{Qualitative clip organization}
The evaluated clips were organized as follows:
\begin{flushleft}
\ttfamily\footnotesize
spatial/\\
\quad source/\{base,videodpo,vari\}.mp4\\
\quad equivalent/\{base,videodpo,vari\}.mp4\\[2pt]
action\_binding/\\
\quad source/\{base,videodpo,vari\}.mp4\\
\quad equivalent/\{base,videodpo,vari\}.mp4
\end{flushleft}
All clips use the compared methods and the shared 720p generation protocol described above.
\end{analysisbox}

\paragraph{Spatial-relation case.}
Across both forms, Base and VideoDPO place the reader on the bench, replacing the requested \emph{in-front-of}/\emph{behind} relation with a frequent reading-on-a-bench configuration. One baseline clip additionally duplicates the reader, but the shared failure is the relation change rather than the mere presence of a counting artifact. The two \method{} clips preserve the intended relation throughout the sequence while allowing an unconstrained posture change: the boy sits on the grass in one form and stands in front of the bench in the other. The complete clips also expose temporal behavior hidden by the three-frame layout. Baseline book pages can flip repeatedly or change geometry abruptly, whereas the \method{} outputs maintain a more coherent reading interaction; the seated result contains a brief localized page turn and the standing result remains stable without changing the spatial arrangement.

\paragraph{Action-binding case.}
The baseline videos display several ways in which a partially correct scene can still violate the compositional intent. The man--newspaper event is often present, but squirrels may be duplicated, appear on the bench rather than the tree, or fail to maintain a single upward trajectory; one clip instead duplicates the reader. The two \method{} clips preserve both assignments over time: the man continues reading and the squirrel remains on the tree while moving upward. The full clips also show stable newspaper handling with limited page motion, rather than a still-frame coincidence, and make the climbing progression visible from low to higher positions on the trunk.

Checkpoint selection precedes qualitative inspection. We then select held-out groups with clear weakest-form improvements and interpretable failure modes; the quantitative conclusions are based on the full benchmark results.

\subsection{Scope of the public preprint}
This preprint includes the full technical appendix and static qualitative panels. The twelve complete example clips and the project code/data are not distributed with it. The project implementation includes \method{} schemas, the supervision compiler, GCD and FC-DPO objectives, robustness aggregation, prompt templates, deterministic synthetic fixtures, unit and smoke tests, a paper-to-code mapping, and manuscript-aggregate arithmetic audits. Documented adapters connect to external backbones, evaluator repositories, and pretrained weights.

\begin{figure*}[!t]
\centering
\includegraphics[width=0.985\textwidth]{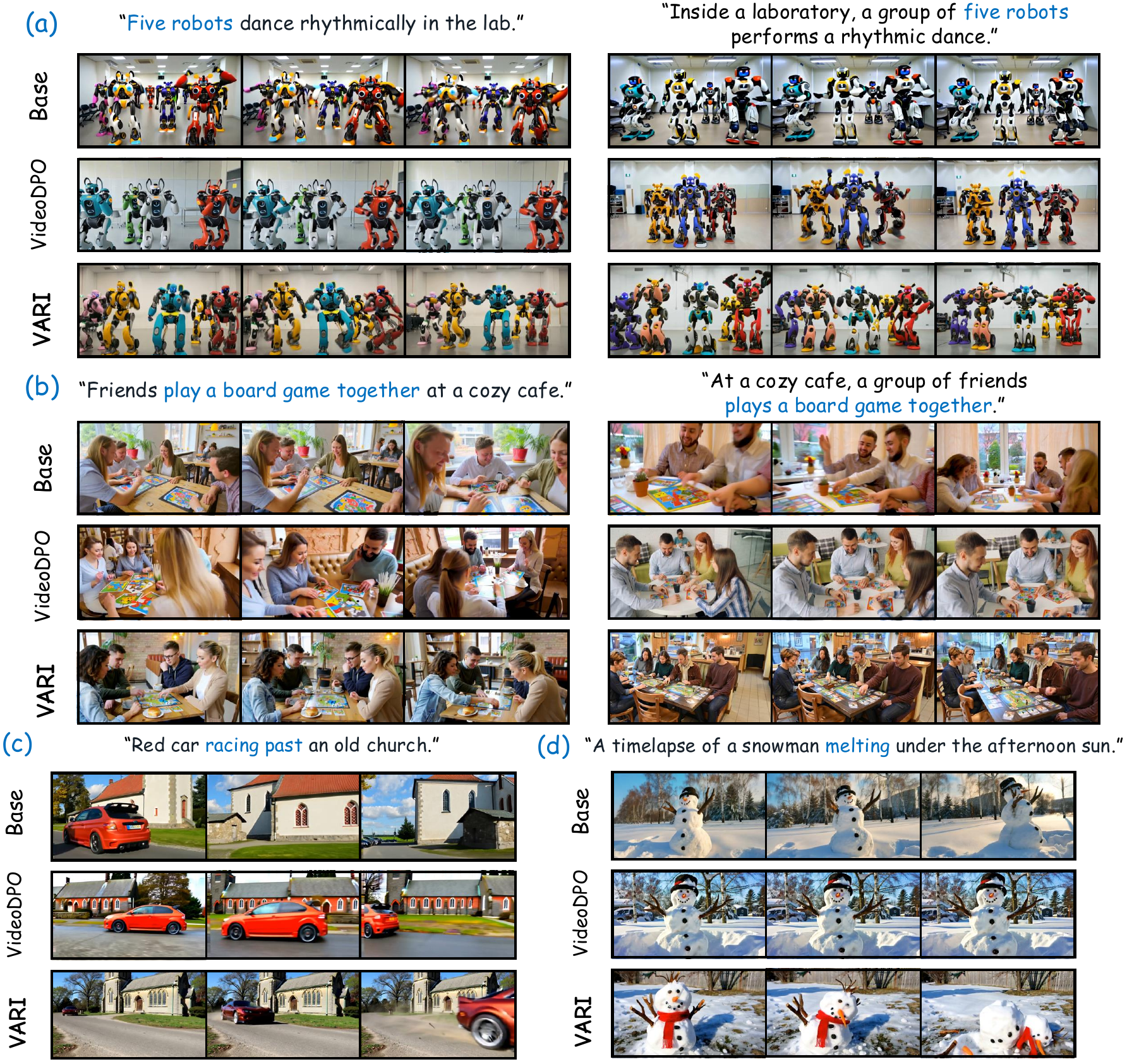}
\vspace{-5pt}
\caption{Four additional held-out cases with ordered frames. (a) \method{} preserves five salient robots and the group dance across both forms. (b) It maintains one coherent game layout and stable participants. (c) It moves the car past a stable church. (d) It realizes progressive snowman collapse and volume loss.}
\label{fig:app-qualitative}
\end{figure*}

\begin{table}[H]
\centering
\scriptsize
\setlength{\tabcolsep}{2.8pt}
\begin{tabular}{@{}>{\raggedright\arraybackslash}p{0.31\linewidth}c>{\raggedright\arraybackslash}p{0.39\linewidth}@{}}
\toprule
\textbf{Comparison} & \textbf{Forms} & \textbf{Material shown} \\
\midrule
Main-paper spatial relation & $2\times3$ & Displayed strips (6 evaluated clips) \\
Main-paper action binding & $2\times3$ & Displayed strips (6 evaluated clips) \\
Appendix paired-form cases & $2\times2$ & Static temporally ordered panels \\
Appendix dynamics cases & $2\times1$ & Static temporally ordered panels \\
\bottomrule
\end{tabular}
\caption{Qualitative evidence shown in this preprint. Each count is prompt forms $\times$ compared methods or intents, as applicable.}
\label{tab:qualitative-material}
\end{table}

\subsection{Additional qualitative results}
\label{sec:additional-qualitative}
Figure~\ref{fig:app-qualitative} provides four additional held-out comparisons using temporally ordered frames. Cases (a) and (b) compare the source prompt with one verified semantic-preserving form, while cases (c) and (d) focus on motion realization and temporal state change from the source prompt. These four cases are presented as static appendix panels; the twelve full clips underlying the main-paper frame strips are not distributed with this preprint.

\paragraph{Numeracy and group action.}
In panel (a), the numeral is a binding constraint rather than a stylistic preference. Base produces crowded groups with additional background robots, while VideoDPO commonly exposes only four dominant robots or changes the apparent count across frames and forms. \method{} keeps five clearly identifiable robots in both prompt formulations and preserves the requested collective dancing behavior with stronger temporal and appearance consistency.

\paragraph{Multi-person interaction and object coherence.}
Panel (b) stresses scene-level logic rather than a single highlighted token. Base and VideoDPO frequently turn the singular board game into multiple disconnected, duplicated, or fragmented playing surfaces, weakening both object identity and the social interaction. \method{} produces the most coherent table geometry and game layout, keeps the participants engaged with the same central activity, and yields cleaner faces, hands, and cafe structure across the sequence.

\begin{analysisbox}{Qualitative audit cues}
The four appendix cases are read against explicit visual commitments: exact entity count and shared group action in panel (a); stable participants and one coherent game object in panel (b); object displacement relative to fixed landmarks in panel (c); and a sustained, monotonic state transition in panel (d). The temporally ordered frames support these checks without substituting for the complete clips used in the main-paper cases.
\end{analysisbox}

\paragraph{Motion realization and scene preservation.}
For panel (c) in Figure~\ref{fig:app-qualitative}, the decisive requirement is object motion relative to fixed landmarks. The Base strip provides little evidence of the car translating through the scene and is dominated by viewpoint change. VideoDPO keeps a visible car but shows weaker frame-to-frame appearance, scale, and trajectory continuity. \method{} presents a clear progression from a distant car to a close pass while retaining the old church as a stable spatial reference.

\paragraph{Temporal state change.}
Panel (d) in Figure~\ref{fig:app-qualitative} tests whether the model executes a sustained transformation rather than merely rendering the correct initial object. The Base snowman remains essentially intact, and VideoDPO shows at most a small local dripping cue without a substantial change of state. \method{} produces a pronounced, monotonic melting sequence: the body loses volume, the structure collapses, and the final frame visibly differs from the initial snowman.

\endgroup
\end{document}